%% file: main.tex
\documentclass[10pt]{article}
\usepackage[utf8]{inputenc}
\usepackage[a4paper, left=2cm, right=2cm, top=2cm, bottom=3cm]{geometry}
\usepackage{amsmath,amssymb,amsthm,mathrsfs,nicefrac}
\usepackage{mathtools}
\usepackage{latexsym,amscd,amsbsy,dsfont,amsfonts}
\usepackage{graphicx,wrapfig}
\usepackage{authblk,textcomp}
\usepackage[dvipsnames]{xcolor}
\usepackage{cancel}

\usepackage{bm}
\usepackage[cal=euler]{mathalfa}
\usepackage{libertine}
\usepackage[libertine,smallerops]{newtxmath}
\usepackage[T1]{fontenc}
\usepackage[font=small]{caption}
\usepackage{tikz}
\usetikzlibrary{arrows}
\usetikzlibrary{decorations}
\usepackage{subcaption}
\usepackage{booktabs}
\usepackage{doi}
\usepackage{float}
\usepackage{enumitem}
\usepackage{physics}
\usepackage{hyperref}
\hypersetup{pdfauthor={IdePHICS},pdftitle={PhaseDiagram},%
            colorlinks, linktocpage=true, pdfstartpage=1, pdfstartview=FitV,%
    breaklinks=true, pdfpagemode=UseNone, pageanchor=true, pdfpagemode=UseOutlines,%
    plainpages=false, bookmarksnumbered, bookmarksopen=true, bookmarksopenlevel=1,%
    hypertexnames=true, pdfhighlight=/O,%
    urlcolor=orange, linkcolor=blue, citecolor=NavyBlue}

\input{setting_customs}

\usepackage[toc,page]{appendix} 
\appto\appendix{\counterwithin{equation}{section}}

\title{Phase diagram of Stochastic Gradient Descent in high-dimensional two-layer neural networks}

\author[1,3]{\vspace{2.5ex} Rodrigo Veiga\thanks{Currently: Ecole Polytechnique F\'{e}d\'{e}rale de Lausanne (EPFL), Lab for Statistical Mechanics of Inference in Large Systems (SMILS), CH-1015 Lausanne, Switzerland. Email:  \href{mailto:rodrigo.veiga@epfl.ch}{rodrigo.veiga@epfl.ch}}}
\author[1]{Ludovic Stephan}
\author[1]{Bruno Loureiro}
\author[1]{Florent Krzakala}
\author[2]{Lenka Zdeborov\'a}
\affil[1]{\small Ecole Polytechnique F\'{e}d\'{e}rale de Lausanne (EPFL).
Information, Learning and Physics (IdePHICS) lab. \newline CH-1015 Lausanne, Switzerland.}
\affil[2]{\small Ecole Polytechnique F\'{e}d\'{e}rale de Lausanne (EPFL).
Statistical Physics of Computation (SPOC) lab. \newline CH-1015 Lausanne, Switzerland.}
\affil[3]{\small Universidade de S\~{a}o Paulo. Instituto de F\'{i}sica. S\~{a}o Paulo, SP, Brazil.}

\date{}

\begin{document}

\maketitle

\vspace{-2.5ex}
\begin{abstract}
Despite the non-convex optimization landscape, over-parametrized shallow networks are able to achieve global convergence under gradient descent. The picture can be radically different for narrow networks, which tend to get stuck in badly-generalizing local minima. Here we investigate the cross-over between these two regimes in the high-dimensional setting, and in particular investigate the connection between the so-called mean-field/hydrodynamic regime and the seminal approach of Saad \& Solla. Focusing on the case of Gaussian data, we study the interplay between the learning rate, the time scale, and the number of hidden units in the high-dimensional dynamics of stochastic gradient descent (SGD). Our work builds on a deterministic description of SGD in high-dimensions from statistical physics, which we extend and for which we provide rigorous convergence rates.
\end{abstract}




\section{Introduction}

Descent-based algorithms such as stochastic gradient descent (SGD) and its variants are the workhorse of modern machine learning. They are simple to implement, efficient to run and most importantly: they work well in practice. 
A detailed understanding of the performance of SGD is a major topic in machine learning.
Quite recently, significant progress was achieved in the context of learning in shallow neural networks. In a series of works, it was shown that the optimisation of wide two-layer neural networks can be mapped to a convex problem in the space of probability distributions over the weights  \cite{mei_2018,chizat_2018,rotskoff_2019,sirignano2020mean}. This remarkable result implies global convergence of two-layer networks towards perfect learning provided that the number of hidden neurons is large, the learning rate is sufficiently small and enough data is at disposition. This line of work is commonly referred to as the mean-field or the hydrodynamic limit of neural networks. Mathematically, these works showed that one could describe the entire dynamics
using a partial differential equation (PDE) in $\inp$ dimensions.

In a different, and older, line of work one-pass SGD for two-layer neural networks with a {\it finite number} $\hids$ of hidden units, synthetic Gaussian input data and teacher-generated labels has been widely studied starting with the seminal work of \cite{saad_1995}. These works consider the limit of high-dimensional data and show, in particular, that the stochastic process driven by gradient updates converge to a set of $\hids^2$ deterministic ordinary differential equations (ODEs) as the input dimension $\inp\to\infty$ and the learning rate is proportional to $1/ \inp$. The validity of these ODEs in this limit was proven by \cite{goldt_2019}. However, the picture drawn from the analysis of these ODEs is slightly different from the mean-field/hydrodynamic picture: in this case SGD can get stuck for long time in minima associated to no specialization of the hidden units to the teacher hidden units, and even when it converges to specializing minima, it fails to perfectly learn (i.e. to achieve zero population risk). In fact, in this analysis, the interplay between the limit of the learning rate going to zero and $\inp \to \infty$ appeared to be fundamental. 

One should naturally wonder about the link between these two sets of works with, on the one hand a $\inp$-dimensional PDE (with large $\hids$), and on the other a $\hids^2$-dimensional ODE (with large $\inp$). In this work we aim to build a bridge between these two approaches for studying one-pass SGD. 

Our starting point is the framework from \cite{saad_1995}, which we build upon and expand to a much broader range of choices of learning rate, time scales, and hidden layer width. This allows us to provide a sharp characterisation of the performance of SGD for two-layer neural networks in high-dimensions. We show it depends on the precise way in which the limit is taken, and in particular on how the quantity of data, the hidden layer width, and the learning rate scale as $\inp \to \infty$. For different choices of scaling, we can observe scenarios such as perfect learning, imperfect learning with an unavoidable error, or even no learning at all. 

 

 \begin{figure}[t]
\begin{subfigure}{.45\textwidth}
  \centering
    \includegraphics[width=0.95\linewidth]{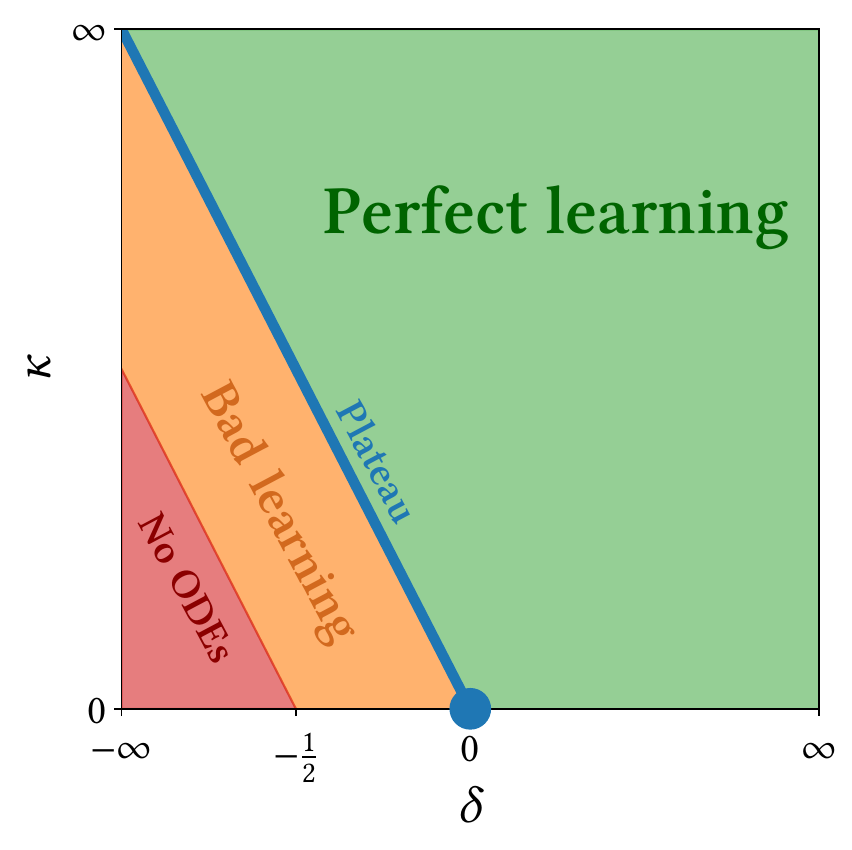}
   \caption{The phase diagram of SGD learning regimes for two-layer neural networks in the high-dimensional input layer limit $\inp \rightarrow\infty$. Eqs. \eqref{eq:scalings} define proper time scalings for each of the regions. Perfect learning region: $\kap + \del > 0 $. Plateau line: $\kap+\del=0$. Bad learning region: $-\nicefrac{1}{2} < \kap + \del < 0$. No ODEs region: $ \kap+\del < - \nicefrac{1}{2}  $.}     \label{fig:phase_diagram}
\end{subfigure}%
\hfill
\begin{subfigure}{.45\textwidth}
  \centering
\centerline{\includegraphics[width=\textwidth]{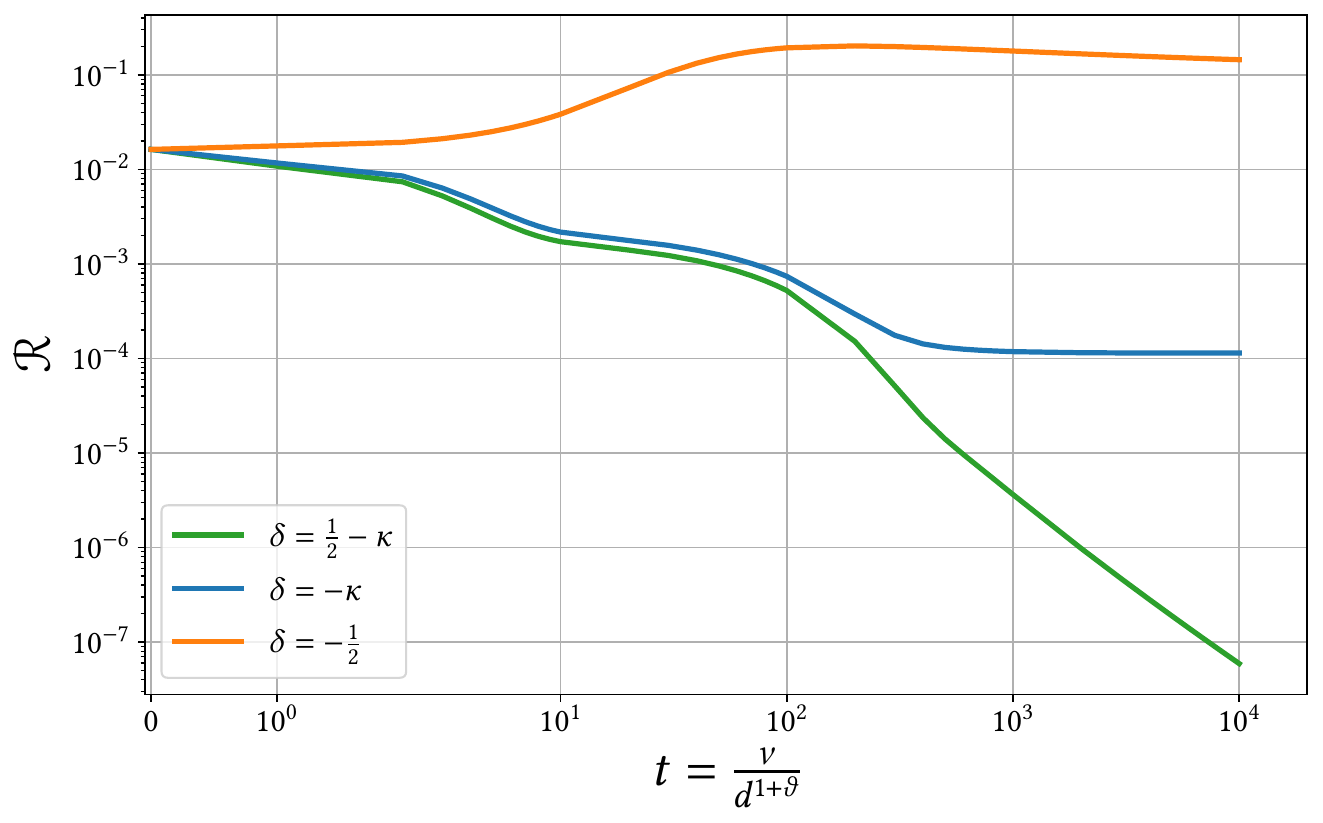}}
\caption{A solution of the ODEs in all regions of Figure \ref{fig:phase_diagram}, with matching colors. Parameters $\kap = 0.301$, $\hids=8$, $\hidt=4$, $\prs = \delta_{rs}$. Noise: $\noise = 10^{-3}$. Activation function: $\act(x) = \erf(x/\sqrt{2})$. Data distribution: $\Prob (\x) =  \gauss(\x | \bm{0}, \Id )$.  The time scaling is not uniform through the phase diagram: $\vartheta = \kap+\delta$ on green and blue regimes and $\vartheta = 2(\kap+\delta)$ on the orange region. The green curve decays as a power law  to zero excess error.}\label{fig:simul_phdiagr__}
  \end{subfigure}
\caption{Phase diagram (left) and typical behavior of the ODE in each regions (right).}
\end{figure}
 
As a consequence of our analysis, we provide a phase diagram (see Figure~\ref{fig:phase_diagram}) describing the possible scenarios arising in the high-dimensional setting.
Our {\bf main contributions} are as follow:
\begin{description}[wide = 2pt]
    \item[C1]  We rigorously show that the dynamics of SGD can be captured by a set of deterministic ODEs, considerably extending the proof of \cite{goldt_2019} to accommodate for general time scalings defined by an arbitrary learning rate, and a general range of hidden layer width. We provide much finer non-asymptotic guarantees which are crucial for our subsequent analysis. 
    \item[C2]  From the analysis of the ODEs, we derive a phase diagram of SGD for two-layer neural networks in the high-dimensional input layer limit $\inp \rightarrow\infty$. In particular, scaling both the learning rate $\lr$ and hidden layer width $\hids$ with the input dimension $\inp$ as 
\begin{subequations}
\label{eq:scalings}
\begin{equation}
    \lr \propto d^{-\del} \;,
\end{equation}
\begin{equation}
    \hids \propto  d^\kappa  \;,
\end{equation}
\end{subequations}
we identify four different learning regimes which are summarized in Figure \ref{fig:phase_diagram}:
\begin{itemize}
    \item Perfect learning (\textcolor{OliveGreen}{green region, $\kap>-\del$}): we show that perfect learning (zero population risk) can be asymptotically achieved 
    with $\nsamp \sim \inp^{1+\kap+\del}$ samples
    even for tasks with additive noise. 
    \item Plateau (\textcolor{RoyalBlue}{blue line $ \kap = -\del $}): learning reaches a plateau related to the noise strength. The point $\kap=\del=0$ goes back to the classical work of \cite{saad_1995}.
    \item Bad learning (\textcolor{orange}{orange region $  -\nicefrac{1}{2} < \kap + \del < 0   $}): here the noise dominates the learning process. 
    \item No ODEs (\textcolor{red}{red region  $ \kap+\del < - \nicefrac{1}{2}$}): the stochastic process associated to SGD is not guaranteed to converge to a set of deterministic ODEs. This region is thus outside the scope of our analysis. 
\end{itemize}
\end{description}

To better illustrate this phase diagram we present in Figure \ref{fig:simul_phdiagr__} a solution of the ODEs in all three regimes.

\paragraph{Relation to previous work --} 
Deterministic dynamical descriptions of one-pass stochastic gradient descent in high-dimensions have a long tradition in the statistical physics community, starting with single- and two-layer neural networks with few hidden units \cite{kinzel_1990,kinouchi_1992, copelli_1995, biehl_1995, riegler_1995}. The seminal work by \cite{saad_1995} overcame previous limitations by constructing a set of deterministic ODEs for two-layer networks with any finite number of hidden units, paving the way for a series of important contributions \cite{saad_1996, vicente_1998,saad_1999, goldt_2019}. This line of work corresponds to the $\kappa = \delta = 0$ case of Figure \ref{fig:phase_diagram}. One of our goal is to generalize this picture beyond fixed hidden layer size and learning rate.

A more recent line of work investigating the dynamics of SGD is the so-called \emph{mean-field limit} \cite{mei_2018,mei_2019, chizat_2018,rotskoff_2019,sirignano2020mean}, which connects the SGD dynamics of large-width two-layer neural networks to a diffusion equation in the hidden layer weight density. In particular, \cite{mei_2019} provide non-asymptotic convergence bounds for sufficiently small learning rates, corresponding to the green region of Figure \ref{fig:phase_diagram} (with $ \hids \to \infty$). The mean-field approach computes the empirical distribution (in $\R^\inp$) of the hidden layer weights, while we focus on the macroscopic overlaps between the teacher and student weights.

\paragraph{Reproducibility} A code is provided at \href{https://github.com/rodsveiga/phdiag_sgd}{https://github.com/rodsveiga/phdiag\_sgd}.

\section{Setting}
Consider a supervised learning regression task. The data set is composed of $\nsamp$ pairs $ (\x^\i , y^\i)_{\i \in [\nsamp]} \in \R^{\inp + 1} $ identically and independently sampled from $ \Prob (\x , y)$. The probability $ \Prob (\x) $  is assumed to be known and $ \Prob (y | \x) $ is modelled by a two layer neural network called the {\it teacher}. Given a feature vector $\x^\i \in \R^\inp$, the respective label $ y^{\i} \in \R$ is defined as the output of a network with $\hidt$ hidden units, fixed weights $ \W^* \in \R^{ \hidt \times \inp }$ and an activation function $\act : \R \rightarrow \R$:
\begin{equation}
    y^\i = \f ( \x^\i  , \W^* )  + \sqrt{\noise} \noisevar^\i   \;,
\end{equation}
where \begin{equation}
    \f ( \x^\i , \W^* ) = \frac{1}{\hidt} \sum_{r=1}^{\hidt}  \act \left( \frac{\w_{r}^{*\top} \x^\i}{\sqrt{\inp}}  \right) = \frac{1}{\hidt}
     \sum_{r=1}^{\hidt} \sigma ( \lf_{r}^{*\i}    ) \;,
\end{equation}
with $\w_{r}^* \equiv [\W^* ]_r \in \R^\inp$ as the $r$-th row of the matrix $\W^*$ and $\lf^{*\i}_{r} \equiv \nicefrac{\w_{r}^{*\top} \x^\i}{\sqrt{\inp}} \in \R  $ as the $r$-th component of the teacher {\it local field} vector $ \bm{\lf}^{*\i} \in \R^\hidt $. The parameter $\noise \ge 0$ controls the strength of additive label noise: $\noisevar^\i  \sim \Prob ( \noisevar^\i) $ such that $ \Expnoise [\noisevar] = 0  $ and $ \Expnoise [\noisevar^2] = 1$.

Given a new sample $ \x \sim \Prob (\x) $ outside the training data, the goal is to obtain an estimation $\hat{\f} ( \x )$ for the respective label~$y$. The error is quantified by a loss function $\loss ( y , \hat{\f} ( \x, \bm{\Theta} ) )$, where $\bm{\Theta}$ is an arbitrary set of parameters to be learned from data.

In this manuscript we are interested in the problem of estimating $\W^*$ with another two-layer neural network with the same activation function, which we will refer to as the {\it student}. The student network has $\hids$ hidden units and a matrix of weights $\W \in \R^{\hids \times \inp}$ to be {\it learned} from the data. Given a feature vector $\x \sim \Prob (\x)$ the student prediction for the respective label is given as
\begin{equation}
   \hat{\f} ( \x , \W) = \frac{1}{\hids}  \sum_{j=1}^{\hids} \act \left( \frac{\w_{j}^\top \x}{\sqrt{\inp}}  \right) =   \frac{1}{\hids} \sum_{j=1}^{\hids} \act ( \lf^\i_j )\;,
\end{equation}
where $\w_{j} \equiv [\W]_j \in \R^\inp$ is the $j$-th row of the matrix $\W$ and $\lf_{j} \equiv \nicefrac{\w_{j}^\top \x}{\sqrt{\inp}} \in \R $ is defined as $j$-th component of the student {\it local field} vector $ \bm{\lf} \in \R^\hids $.
\paragraph{One-pass gradient descent --} 
Typically, one minimizes the {\it empirical risk} over the full data set. Instead, learning with {\it one-pass} gradient descent minimizes directly the {\it population risk}:
\begin{equation}
    \risk ( \W,  \W^*  ) \equiv  \Expxy \left[ \loss \left( \f( \x, \W^* ), \hat{\f} ( \x, \W ) \right)    \right]    \;.
\end{equation}
Given a {\it single} sample $(\x^\i , y^\i) $ the weights are updated sequentially by the gradient descent rule:
\begin{equation}\label{eq:w_update}
   \w^{\i+1}_{j} = \w^{\i}_{j} - \lr \bm{\nabla}_{\w_j}\loss \left( y^\i , \hat{\f} ( \x^{\i}, \bm{W} ) \right)   \;,
\end{equation}   
with $\i \in [\nsamp]$ and $j \in [\hids]$. The parameter $\lr > 0 $ is the learning rate. Despite being a simplification with respect to batch learning, one-pass gradient descent is an amenable surrogate for the theoretical analysis of non-convex optimization, since at each step the gradient is computed with a fresh data sample, which is equivalent to performing SGD directly on the population risk.


In particular, in this manuscript we assume realizability $p \ge k$, and focus our analysis on the square loss $ \loss ( y , \hat{y}  ) = \frac{1}{2} (  y - \hat{y})^2 $, leading to
\begin{equation}
 \w^{\i +1}_j = \w^{\i}_{j} + \frac{\lr}{\hids\sqrt{\inp}} \act' (\lf_{j}^\i) \Er^\i \x^\i  \;, 
\end{equation}
where \begin{equation}
    \Er^\i \equiv \frac{1}{\hidt} \sum_{r=1}^{\hidt} \act( \lf_{r}^{*\i}  ) - \frac{1}{\hids} \sum_{l=1}^{\hids} \act ( \lf_{l}^\i  )   + \sqrt{\noise} \noisevar^\i \;.
\end{equation}
with population risk given by
\begin{equation}
    \risk ( \W,  \W^*  ) =  \frac{1}{2} \Expxy \left[ \left( \hat{\f}( \x ,  \W) - \f ( \x ,  \bm{\
    W}^*)  \right)^{2}   \right]    \;.
\end{equation}
Therefore, from the above expression we can see that to monitor the population risk along the learning dynamics it is sufficient to track the joint distribution of the local fields $(\bm{\lf} , \bm{\lf}^*)$. For Gaussian data $ \Prob (\x) =  \gauss(\x | \bm{0}, \Id ) $, one can replace the expectation $\Expxy [ \cdot ]$ by $\Explf [ \cdot ]$ and fully describe the dynamics through the following sufficient statistics, known in the statistical physics literature as {\it macroscopic variables}: 
\begin{subequations}
\begin{equation}
   \Q^\i \equiv  \Expxy \left[ \bm{\lf}^\i \bm{\lf}^{\i\top}  \right] = \frac{1}{\inp} \W^{\i\top} \W^{\i} \;,
\end{equation}
\begin{equation}
   \M^\i \equiv  \Expxy \left[ \bm{\lf}^\i \bm{\lf}^{*\i\top}  \right] 
   =  \frac{1}{\inp} \W^{\i\top} \W^* \;,
\end{equation}
\begin{equation}
   \P \equiv \Expxy \left[ \bm{\lf}^{*\i} \bm{\lf}^{*\i\top}  \right] = \frac{1}{\inp} \W^{*\top} \W^* \;.
\end{equation}
\end{subequations}
with matrix elements, called {\it order parameters} in the statistical physics literature, denoted by $ \qjl^\i \equiv [\Q^\i]_{jl} $,  $ \mjr^\i \equiv [\M^\i]_{jr} $ and  $ \prs
\equiv [\P]_{rs} $. The {\it macroscopic state} of the system at the learning step $\i$ is given by the {\it overlap matrix} $\Om^\i \in \R^{(\hids+\hidt)\times(\hids+\hidt)} $:
\begin{equation}
\Om^\i = \begin{bmatrix}
\Q^\i & \M^\i \\
\M^{\i\top} & \P 
\end{bmatrix} \;,
\end{equation}
and the population risk is completely determined by the macroscopic state:
\begin{equation}
\label{eq:pop_risk}
    \risk ( \Om  ) =    \frac{1}{2} \Explf \Expnoise 
    \left[ \left( \hat{\f}( \bm{\lf} ) - \f ( \bm{\lf}^* )  \right)^{2}   \right]    \;.
\end{equation}
The training dynamics \eqref{eq:w_update} defines a discrete-time stochastic process for the evolution of the overlap matrix
\begin{equation}
\label{eq:disc_seq}
\left\{  \Om^\i \in \R^{(\hids+\hidt)\times(\hids+\hidt)} \; , \i \in [ \nsamp  ] \right\} \;,
\end{equation}
with $\P$ fixed and $\Q^\i$ and $\M^\i$ updated as:
\begin{subequations}
\label{eq:qm_up}
\begin{equation}
\qjl^{\i+1} - \qjl^{\i} = \frac{\lr}{\hids\inp} \underbrace{\left( \Er_{j}^\i \lf_{l}^\i  + \Er_{l}^\i \lf_{j}^\i \right)}_{\text{learning}} + \frac{\lr^2 \lVert \x \rVert^2 }{\hids^2 \inp^2 } 
\underbrace{\Er_{j}^\i   \Er_{l}^\i}_{\text{variance}}    \;,
\end{equation}
\begin{equation}
\mjr^{\i+1} - \mjr^\i = \frac{\lr}{\hids\inp} \underbrace{\Er_{j}^\i \lf_{r}^{*\i}}_{\text{learning}} \;, 
\end{equation}
\end{subequations}
with $\i \in [\nsamp]$, $j , l \in [\hids]$, $r \in [\hidt]$ and $\Er_{j}^\i \equiv \act'( \lf_{j}^\i ) \Er^\i$. In what follows, we will make the concentration assumption $\lVert \x \rVert^2 = d$; this will be justified in the proof of Theorem \ref{th:conv_eps}.

We emphasize in \eqref{eq:qm_up} the specific role played by each term in the right hand-side. The "learning" terms are the fundamental ones, that actually drive the learning of the teacher by the student.  We show in Appendix \ref{app:c:flow} that these "learning" terms are identical to those obtained in the gradient flow approximation of SGD, whose performance is the topic of many works \cite{mei_2018,chizat_2018,rotskoff_2019,sirignano2020mean}.  Those are {\it precisely} the terms that draw the population risk towards zero. However, in our setting there is an additional variance term (so that this flow approximation is incomplete) that corresponds to the fluctuations of $\loss( \x, \W, \W^*)$ around its expected value $\risk(\W, \W^*)$. In particular, this is where the effects of the noise $\noisevar$ can be felt. These terms were sometimes denoted as ($I_2$) and ($I_4$) in \cite{saad_1995_0}. We shall see that the additional "variance" term is the one responsible for the plateau in the critical (blue) region of Figure \ref{fig:phase_diagram}, while its contribution vanishes in the perfect learning (green) region.

Additionally, albeit our work particularizes to Gaussian input data, we believe our conclusion, and the phase diagram discussed in Figure~\ref{fig:phase_diagram}, to hold beyond this restricted case. Indeed, while the Gaussian assumption is crucial to reach a particular set of ODEs and their analytic expression, the approach can be applied to more complex data distribution, as long as one can track the sufficient statistics required to have a closed set of equations. For instance, \cite{refinetti2021classifying} obtained very similar equations for an arbitrary mixture of Gaussians -- that would obey the same scaling analysis as ours -- while  \cite{goldt2020gaussian,hu2020universality,montanari2022universality} proved that many complex distributions behave as Gaussians in high-dimensional setting, including, e.g. realistic GAN-generated data. We thus expect our conclusions to be robust in this respect.

\section{Main results}\label{sec:main}
Although $\alp_0 = \i / \inp$ seems to be the most natural time scaling in the high-dimensional limit $\inp\rightarrow\infty$, if $\lr$ and $\hids$ are allowed to vary with $\inp$ the right-hand side (RHS) of Eqs.~\eqref{eq:qm_up} can diverge and render the ODE approximation obsolete. Instead, for a given time scaling $\delta \alp$, we can rewrite Eqs.~\eqref{eq:qm_up} as
\begin{subequations}
\label{eq:qm_up_rewrite}
\begin{equation}
\label{eq:q_up_rewrite}
\frac{\qjl^{\i+1} - \qjl^{\i}}{ \delta \alp } =  \frac{\lr}{\hids\inp \, \delta \alp} \left( \Er_{j}^\i \lf_{l}^\i  + \Er_{l}^\i \lf_{j}^\i  \right)+ 
\frac{\lr^2 }{\hids^2\,\inp\, \delta \alp}
\Er_{j}^\i   \Er_{l}^\i    \;,
\end{equation}
\begin{equation}
\frac{\mjr^{\i+1} - \mjr^\i}{\delta \alp} = \frac{\gamma}{\hids\inp\, \delta \alp}\Er_{j}^\i \lf_{r}^{*\i} \;.
\end{equation}
\end{subequations}
In Theorem \ref{th:conv_eps} we prove that as $\inp \to \infty$, $\Om^{\i}$ converges to the solution of the ODE:
\begin{equation}
\label{eq:ODE_lemma1}
    \dv{\alp}  \bar{\Om} (\alp) = \psi \left( \bar{\Om}(\alp)   \right) \;,
\end{equation}
where $\psi : \R^{(\hids+\hidt)\times(\hids+\hidt)} \rightarrow \R^{(\hids+\hidt)\times(\hids+\hidt)}$ is the expected value of the RHS of Eqs.~\eqref{eq:qm_up_rewrite}, provided that this solution stays bounded. This enhances the result of \cite{goldt_2019} by providing convergence rates to the ODEs encompassing all scalings adopted hereafter:  

\begin{theorem} [Deterministic scaling limit of stochastic processes] \label{th:conv_eps} Let $\T \in \R $ be the continuous time horizon and $ \delta \alp  = \delta \alp (\inp)$ be a time scaling factor such that the following assumptions hold:
\begin{enumerate}
     \item \label{ass_main:2} the time scaling $\delta \alp$ satisfies for some constant $c$,
        \begin{equation}\label{eq:lower_bound_deltat}
        \delta \alp \geq c\, \max\left( \frac{\lr}{\hids \inp}, \frac{\lr^2}{\hids^2 \inp}  \right)
        \end{equation}
    \item \label{ass_main:1} the activation function $\act$ is $L$-Lipschitz,
    \item \label{ass_main:3} the function $\psi : \R^{(\hids+\hidt)\times(\hids+\hidt)} \rightarrow \R^{(\hids+\hidt)\times(\hids+\hidt)}$ is $L'$-Lipschitz.
\end{enumerate}
Then, there exists a constant $C > 0$ (depending on $c, L, L'$) such that for any  $0 \leq \i \leq  \lfloor \T / \delta \alp \rfloor $, the following inequality holds:
\begin{equation}
\label{eq:conv_bound}
\E\;\norm*{\Om^\i -\bar{\Om} \left( \i \delta \alp \right)  }_\infty \le e^{C\tau} \, \log (\hids) \sqrt{\delta \alp} \;.
\end{equation}
\end{theorem}
Our proof is based on techniques introduced in \cite{wang_2018} (namely, their Lemma 2) which studies a different problem with related proof techniques. The proof involves decomposing $\Om^{\i+1}$ as 
\begin{equation}
     \Om^{\i+1} = \Om^\i +  \delta \alp \; \psi(\Om^\i) + \left(\Om^{\i+1} - \Om^\i - \delta \alp \psi(\Om^\i)\right) \;,
\end{equation}
where the two first terms can be considered as a deterministic discrete process, and the last term is a martingale increment. The main challenge lies in showing that the martingale contribution stays bounded throughout the considered time period. 

Although the method is similar to \cite{goldt_2019}, there are a number of differences between the two approaches. First, our proof fixes a number of holes in \cite{goldt_2019}, in particular bounding $q_{jj}^\i$ by a sufficiently slowly diverging function of $\i$. Additionally, the techniques used in this paper yield a dependency in $\hids$ that is nearly negligible, while the previous methods imply bounds that are much too coarse for our needs. 

The function $\psi$ can be computed explicitly for various choices of $\act$, which allows to check Assumption \ref{ass_main:3} directly. We provide in Appendix \ref{app:expec} the necessary computations for $\act(x) = \erf(x/\sqrt{2})$; those for the ReLU unit can be found in \cite{yoshida_2017_statistical}. It can be checked that in the ReLU case, the function $\psi$ is not Lipschitz around the matrices $\Om$ satisfying
\[ \Omega_{jl} = \sqrt{\Omega_{jj} \Omega_{ll}}\]
for some $j \neq l$. However, in every case we have a weaker square-root-Lipschitz property: there exists $C \in  \mathbb R$ such that
\[ \lVert \psi(\Om) - \psi(\Om') \rVert \leq C \left\lVert \sqrt{\Om} - \sqrt{\Om'} \right\rVert \]
for any $\Om, \Om'$. Since the square root function is Lipschitz whenever the eigenvalues of $\Om$ are bounded away from zero (see e.g. \cite{delmoral_2018}), Assumption \ref{ass_main:3} is implied by the condition
\[ \Om^\i \succeq \epsilon I_{p+k} \; ; \]
however, this assumption is much stronger, and becomes unrealistic in the specialization phase (as well as when $p \gg d$).

Theorem \ref{th:conv_eps} allows us to safely navigate through Figure \ref{fig:phase_diagram} by keeping track of convergence rates of the discrete process to a set ODEs. The interplay between learning rate and hidden layer width defines the time scaling $\delta \alp$ and the trade-off between the linear contribution on $\Er_j$ and the quadratic one, playing a central role on whether the network achieves perfect learning or not. Specifically, consider the following learning rate and hidden layer width scaling with $\inp$:
\begin{subequations}
\begin{equation}
    \lr = \frac{\lr_0}{d^\del} \; ,
\end{equation}
\begin{equation}
    \hids = \hids_0  d^\kappa  \;,
\end{equation}
\end{subequations}
where $\lr_0 \in \R^+ $ and $\hids_0 \in \N$ are constants. The exponent $\del \in \R$ can be either greater or smaller than zero, while $\kappa \in \R^+$. Replacing these scalings on Eqs.~\eqref{eq:qm_up}, we find:
\begin{subequations}
\label{eq:qm_general}
\begin{equation}
\qjl^{\i+1} - \qjl^{\i} = \frac{1}{\inp^{1+\kap+\del}} \underbrace{\left(\Er_{j}^\i \lf_{l}^\i  + \Er_{l}^\i \lf_{j}^\i \right)}_{\text{learning}} + \frac{1}{\inp^{1+2(\kap+\del)}}
\underbrace{\Er_{j}^\i   \Er_{l}^\i}_{\text{noise}}    \;,
\end{equation}
\begin{equation}
\mjr^{\i+1} - \mjr^\i  = \frac{1}{\inp^{1+\kap+\del}} \underbrace{\Er_{j}^\i \lf_{r}^{*\i}}_{\text{learning}} \;,
\end{equation}
\end{subequations}
where we have chosen $\lr_0  = \hids_0 $ without loss of generality.

Since the distribution of the label noise $ \Prob (\noisevar) $ is such that $ \Expnoise [\noisevar] = 0  $, the linear contribution in $ \Er_j$ is noiseless in the high-dimensional limit $\inp\rightarrow\infty$, and therefore we will refer to it as the  {\it learning term}. The noise enters in the equations through the variance computed on the quadratic contribution $ \Er_j \Er_l $, which we will refer to as the {\it noise term}; intuitively, it is a high-dimensional variance correction which hinders learning. In order to satisfy \eqref{eq:lower_bound_deltat}, we shall take
\begin{equation}
    \label{eq:dt_max}
     \delta t = \max\left(\frac1{\inp^{1 + \kap + \del}}, \frac1{\inp^{1 + 2(\kap + \del)}} \right)  \;.
\end{equation}
When $\kap + \del \neq 0$, this implies that either the learning term or the noise term scale like a negative power of $\inp$, and is negligible with respect to the other term. It is then easy to check that at a finite time horizon $\T$, the resulting ODEs behave as if the negligible term was not present. We refer to Theorem \ref{th:ode_perturbation} in the appendix for a quantitative proof of this phenomenon. Let us now describe the different regimes depicted in Figure~\ref{fig:phase_diagram}.
\paragraph{Blue line (plateau) --} 
When $\lr$ and $\hids$ are scaled such that $\kap = -\del$,  Eqs.~\eqref{eq:qm_general} converge to 
\begin{subequations}
\label{eq:qmode0}
\begin{equation}
\label{eq:qode0}
    \dv{\qjl}{\alp_0} = \Explf  \left[  \Er_{j} \lf_{l} + \Er_{l} \lf_{j} \right] 
 + \Explf \Expnoise \left[  \Er_{j} \Er_{l}  \right]      \;,
\end{equation}
\begin{equation}
\dv{\mjr}{\alp_0}  =  \Explf \left[  \Er_j \lf_{r}^*  \right]\;,
\end{equation}
\end{subequations}
with $\delta \alp_0 \equiv 1 / \inp $. This regime is an extension of \cite{saad_1995} for which $\kap=\del=0$. The convergence rate to the ODEs scales with $\inp^{-1/2} \log(\inp)$, and the phenomenology we observe for $\kap=\del=0$ is consistent with previous works studying the setting $\kap=\del=0$; namely the existence of an asymptotic plateau proportional to the noise level. For instance, the asymptotic population risk $\risk_\infty$ is known to be proportional to $ \lr \noise$ \cite{goldt_2019} when $\kap=\del=0$ and the dynamics is driven by a rescaled version of Eqs.~\eqref{eq:qmode0}. Since the noise term does not vanish under this scaling, perfect learning to zero population risk is not possible. There is always an asymptotic plateau related to the noise level $\noise$, and the learning rate $\lr$. 

\paragraph{Green region (perfect learning) --} 
If $ \kap > -\del $ we can define the time scaling $\delta \alp_{\kap+\del} \equiv 1 / \inp^{1 + \kap+\del} $. By Theorem \ref{th:conv_eps}, Eqs.~\eqref{eq:qm_general} converge to the following deterministic set of ODEs: 
\begin{subequations}
\label{eq:ode_kappa}
\begin{equation}
\label{eq:qode_kappa}
 \dv{\qjl}{\alp_{\kap+\del}}  =
\Explf  \left[  \Er_{j} \lf_{l} + \Er_{l} \lf_{j} \right]  + \Or \left( \frac{ \Explf \Expnoise \left[  \Er_{j} \Er_{l}  \right]}{\inp^{\kap+\del}}     \right)   \;,
\end{equation}
\begin{equation}
\dv{\mjr}{\alp_{\kap+\del}}  = \Explf  \left[  \Er_j \lf_{r}^*  \right]\;,
\end{equation}
\end{subequations}
at a rate proportional to $ \inp^{-(1+\kap+\del)/2} \log(d) $, where we have highlighted that the noise term vanishes with $  \inp^{-(\kap+\del)} $. Hence, as long as $\kap > - \del$ the noise does not play any role on the dynamics. This setting could be understood by taking an effect learning rate $\lr_{\text{eff}} \propto d^{-\kap-\del} $ on $ \risk_\infty \propto \lr \noise $, which leads to zero population risk, i.e. perfect learning, in the high dimensional limit $\inp\to\infty$. We validate this claim by a finite size analysis in the next section.

As discussed, the time scaling determines the number of data samples required to complete one learning step on the continuous scale. The bigger $\kap+\del$, the more attenuated the noise term, thus the closer to perfect learning. The trade-off is that the bigger $\kap+\del$, the larger the number of samples needed is, since $\nsamp = \tau \inp^{1+\kap+\del}$. Given a realizable learning task, one would thus rather choose the parameters to attain the perfect learning region, but being as close as possible to the plateau line for not increasing too much the needed number of samples. We remark that \cite{mei_2019} provides an alternative deterministic approximation in this regime, with non-asymptotic bounds, whenever $\hids \gg 1$; this is the so-called mean-field approximation, with known convergence guarantees \cite{chizat_2018}.


\paragraph{Orange region (bad learning) --}
We now step in the unusual situation where the learning rate grows faster with $\inp$ than the hidden layer width: $\kap < -\del$. In this case, by \eqref{eq:dt_max} the noise term dominates over the dynamics. Defining the time scaling $\delta \alp_{2(\kap+\del)} \equiv 1 / \inp^{1+ 2(\kap+\del)} $, we have
\begin{subequations}
\label{eq:ode_delkappa}
\begin{equation}
\label{eq:qode_delkappa}
  \dv{\qjl}{\alp_{2(\kap+\del)}}  = \Explf \Expnoise \left[  \Er_{j} \Er_{l}  \right] 
+ \Or \left( \frac{ \Explf  \left[  \Er_{j} \lf_{l} + \Er_{l} \lf_{j} \right] }{\inp^{-(\kap + \del) }}     \right)  \;,
\end{equation}
\begin{equation}
\dv{\mjr}{\alp_{2(\kap+\del)}}  = \Or \left( \frac{\Explf  \left[  \Er_j \lf_{r}^*  \right] }{\inp^{-(\kap+\del)}} \right) \;.
\end{equation}
\end{subequations}

According to Theorem \ref{th:conv_eps} the convergence rate of Eqs.~\eqref{eq:qm_general} to Eqs.~\eqref{eq:ode_delkappa} scales with $\inp^{-(1/2+\kap+\del)} \log(\inp)$. Therefore the existence of the noisy ODEs above is circumscribed to the region 
\begin{equation}
    - \nicefrac{1}{2}  < \kap + \del < 0  \;,
\end{equation}
and presents a convergence trade-off absent in the other regimes:
the faster one of the contributions of Eqs.~\eqref{eq:qm_general} goes to zero, the worse is the convergence rate.
In the present case, the more the learning term is attenuated, i.e. the more negative is $\kap+\del$, the worse the dynamics is described by Eqs.~\eqref{eq:ode_delkappa}.
Although the weights are updated, the correlation between the teacher and the student weights parametrized by the overlap matrix $\M$ remains fixed on its initial value $\M^0$, which is a fixed point of the dynamics under this scaling. Unsurprisingly, this leads to  poor generalization capacity. 


\paragraph{Red region (no ODEs) --}
If $ \kap + \del < - \nicefrac{1}{2} $, the stochastic process driven by the weight dynamics does not converge to deterministic ODEs under the assumptions of Theorem \ref{th:conv_eps}. We are then not able to state any claim about this regime. 

\paragraph{Initialization and convergence --}
There are two additional features worth commenting on the high-dimensional dynamics and its connection to the mean-field/hydrodynamic approach, regarding initialization and the specialization transition. 

In the ODE approach we discuss here, we always observe a first plateau where the teacher-student overlaps are all the same. This means all the hidden layer neurons learned the same linear separator. At this point, the two-layer network is essentially linear. This is called a \emph{unspecialized} network in \cite{saad_1995_0,saad_1995}. In fact, this is a perfectly normal phenomenon, as with few samples even the Bayes-optimal solution would be unspecialized \cite{aubin2019committee}. Only by running the dynamics long enough the student hidden neurons start to \emph{specialize}, each of them learning a different sub-function so that the two-layer network can learn the non-trivial teacher.

Let us make two comments on this phenomenon: (i) while the "linear" learning in the unspecialized regime may remind the reader of the linear learning in the lazy regime \cite{jacot_2018,chizat_2019} of neural nets, the two phenomena are {\it completely} different. In  lazy training, the learning is linear because weights change very little, so that the effective network is a linear approximation of the initial one. Here, instead, the weights are changing {\it considerably}, but each hidden neuron learns essentially the same function.    
(ii) If the ODEs are initialized with weights uncorrelated with the teacher, then the unspecialized regime is a fixed point of the ODEs: the student thus never specializes, at any time. Strikingly, such condition arises as well in the analysis of mean-field equations (see e.g. Theorem 2 in \cite{bach2021gradient} that discusses the need to have \emph{spread} initial conditions with a non-zero overlap with the teacher) to guarantee global convergence. 

This raises the question about the precise dependence of the learning on the initialization condition in the high-dimensional regime, where a random start gets a vanishing ($1/\sqrt{\inp}$) overlap. This is a challenging problem that only recently has been studied (though in a simpler setting) in \cite{tan2019phase, arous2021online,arous2020algorithmic} who showed it yields an additional $\log (\inp)$ time-dependence. Generalizing these results for high-dimensional two-layer nets is an open question which we leave for future work.

\section{Discussion, special cases, and simulations}

To illustrate the phase diagram of Figure~\ref{fig:phase_diagram}, we present now several special cases for which we can perform simulations or numerically solve the set of ODEs. Henceforth, we take $\act(x) \!=\! \erf( x / \sqrt{2}  )$, for which the expectations of the ODEs and of the population risk, Eq.~\eqref{eq:pop_risk}, can be calculated analytically \cite{saad_1995}. The explicit expressions are presented in Appendix~\ref{app:expec}. Teacher weights are such that $\p_{rs} = \delta_{rs}$. The initial student weights are chosen such that the dimension $\inp$ can be varied without changing the initial conditions $ \Q^0$, $\M^0$, $\P$ and consequently the initial population risk $\risk_0$. A detailed discussion can be found in Appendix \ref{app:init_cond}.

\subsection{Saad \& Solla scaling \texorpdfstring{$\kap =\del= 0$}{K = D = 0} }
We start by recalling the well-known setting characterized by the point $\kap =\del= 0 $. The convergence of the stochastic process for fixed learning rate and hidden layer width to Eqs.~\eqref{eq:qmode0} was first obtained heuristically by \cite{saad_1995}. In Figure \ref{fig:dscale} we recall this classical result by plotting the population risk dynamics for different noise levels. Dots represent simulations, while solid lines are obtained by integration of the ODEs, Eq.~\eqref{eq:qmode0}.

\begin{figure}[tb!]
\begin{center}
\centerline{\includegraphics[width=0.45\textwidth]{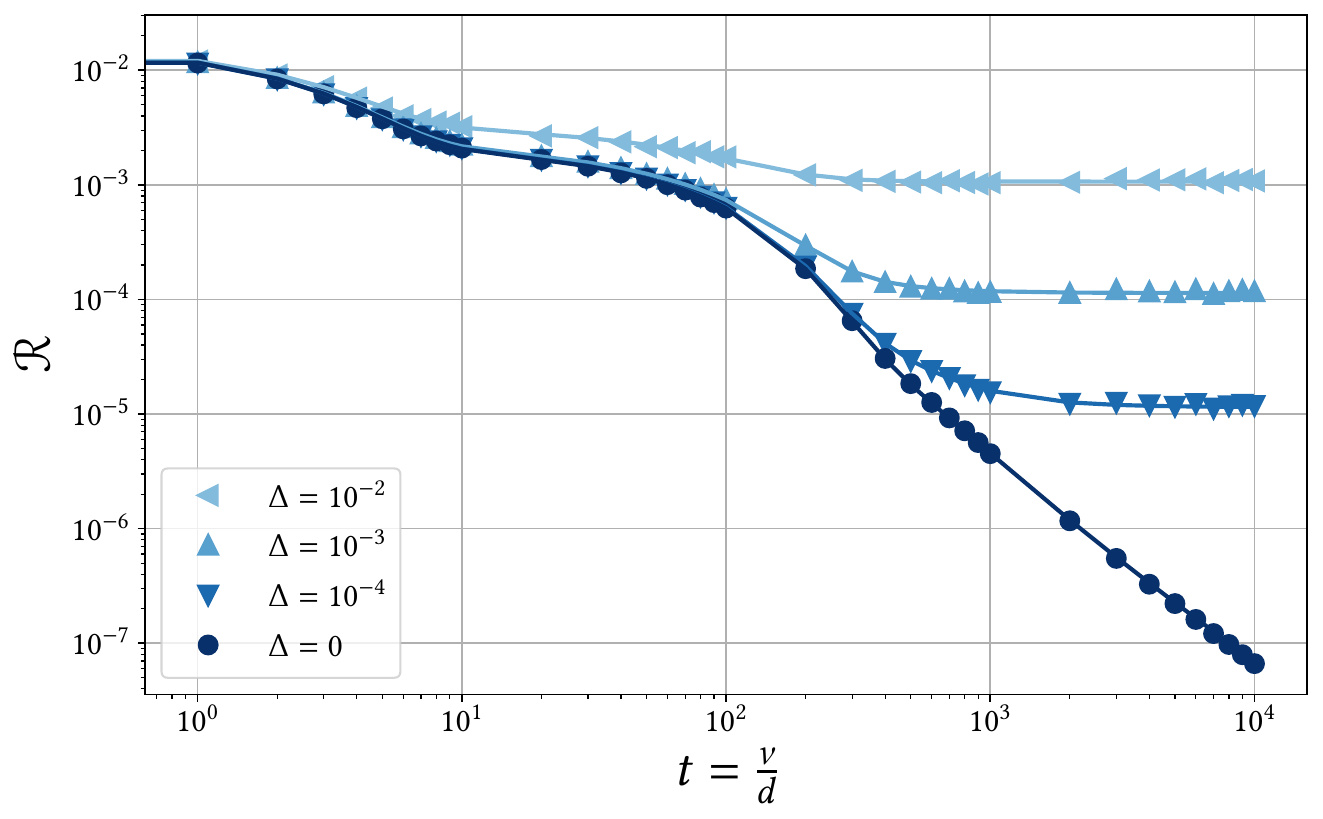}}
\vskip -0.05in
\caption{ Population risk dynamics for $\kap=\del=0$ (Saad \& Solla scaling) : $\hids_0 = 8$, $\hidt=4$, $\prs = \delta_{rs}$. Activation: $\act(x) = \erf(x/\sqrt{2})$. Data distribution: $\Prob (\x) =  \gauss(\x | \bm{0}, \Id )$. Dots represent simulations ($\inp = 1000$), while solid lines are obtained by integration of the ODEs given by Eqs.~\eqref{eq:qmode0}.}
\label{fig:dscale}
\end{center}
\vskip -0.2in
\end{figure}

Learning is characterized by two phases after the initial decay. The first is the unspecialized plateau where all the teacher-student overlaps are approximately the same: $\mjr \approx m $. Waiting long enough, the dynamics reaches the {\it specialization} phase, where the student neurons start to {\it specialize}, i.e., their overlaps with one of the teacher neurons increase and consequently the population risk decreases. This specialization is discussed extensively in \cite{saad_1995}. If $\Delta = 0$, the population risk goes asymptotically to zero. Instead, if $\Delta \ne 0$, the specialization phase presents a second plateau related to the noise $\Delta$.

    The asymptotic population risk $\risk_\infty$ related to the second plateau is proportional to $ \lr \noise$ \cite{goldt_2019} in the high-dimensional limit $\inp \to\infty$ with $\hids$ finite. As mentioned in the previous section, the expectation over  $\Er_j \Er_l$ in Eq.~\eqref{eq:qode0} prevents one from obtaining zero population risk for a noisy teacher. 

\subsection{Perfect learning for \texorpdfstring{$\kap = 0$}{K=0}}
In this section we study the line $\kap = 0$ with $\del >0 $ of Figure \ref{fig:phase_diagram}, for which Eqs.~\eqref{eq:ode_kappa} with $\kap=0$ hold. We show that perfect learning can be asymptotically achieved in the realizable setting for any finite hidden layer width $\hids = \hids_0$. Keeping $\del$ and $\noise$ fixed, we have done simulations increasing the input layer dimension $\inp$. In Figure \ref{fig:d32_vary_d} we set $\del = 1/2$, $\noise = 10^{-3}$ and vary the input layer dimension. The bigger $\inp$ is, the closer we are to the ODE-derived noiseless result.

Gathering the asymptotic population risk from simulations for varying $\inp$ and $\noise$ we perform a finite-size analysis to study the  dependence of $ \risk_\infty$ with $\inp$. This shows that the noise term goes to zero under this setting. In Figure \ref{fig:d32_finite_size} we plot $\risk_\infty$ versus $\inp$ from simulations (dots) for different noise levels. We fit lines under the log-log scale showing that $ \risk_\infty \propto \inp^{-\del}$, as expected. Figure \ref{fig:deps_scal_eps025} draws the same conclusion for $\del=1/4$.

\begin{figure*}[tb!]
\centering
\begin{subfigure}[h]{.45\textwidth}
  \centering
  \includegraphics[width=\linewidth]{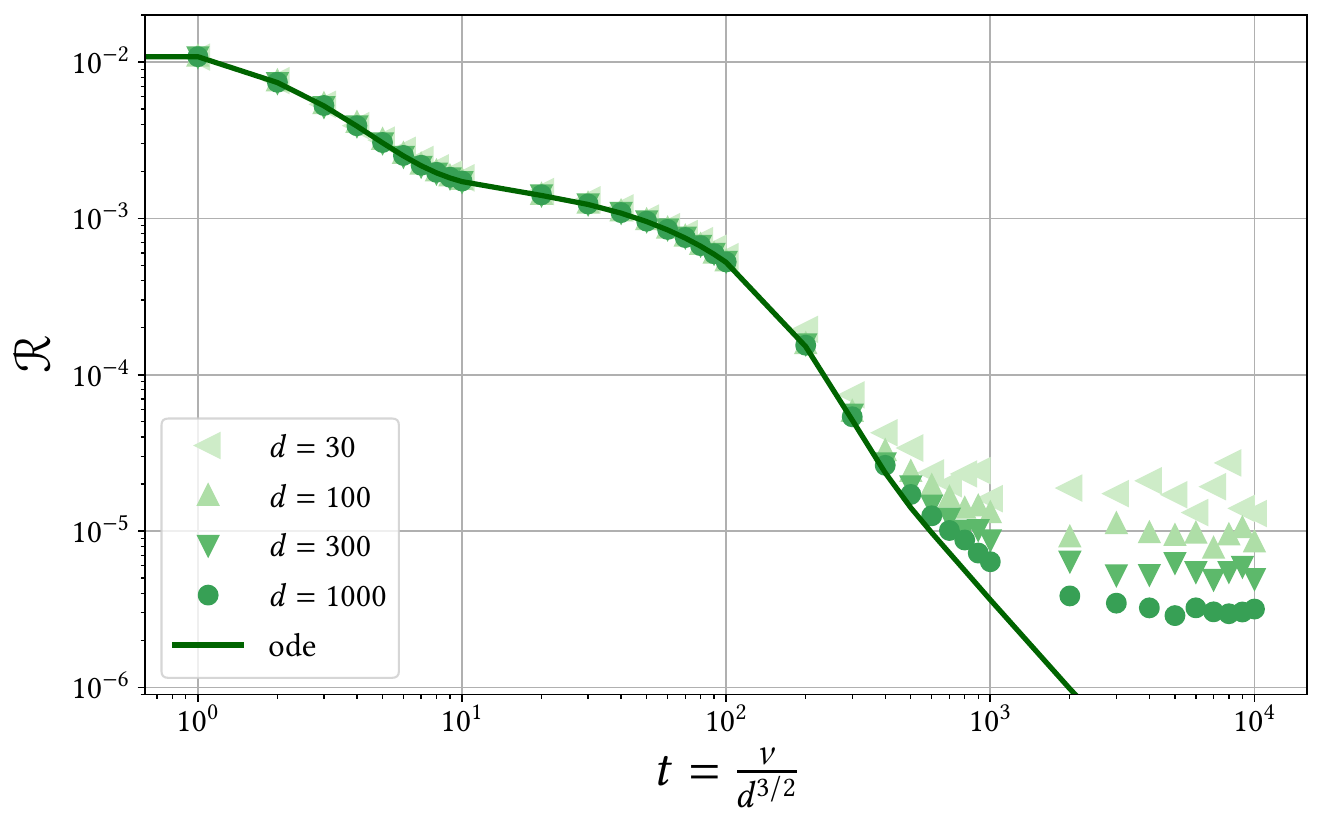}
  \vskip -0.05in
  \caption{Population risk dynamics for $\kap=0$ and $\del=1/2$. Fixed noise $\noise = 10^{-3}$ and varying $\inp$. Dots represent simulations, while the solid line is obtained by integration of the ODEs given by Eqs.~\eqref{eq:ode_kappa}. The data are compatible with the claim that as $\inp \to \infty$ the curve converges to zero population risk.}
  \label{fig:d32_vary_d}
\end{subfigure}%
\hspace{0.5cm}
\begin{subfigure}[h]{.45\textwidth}
  \centering
  \includegraphics[width=\linewidth]{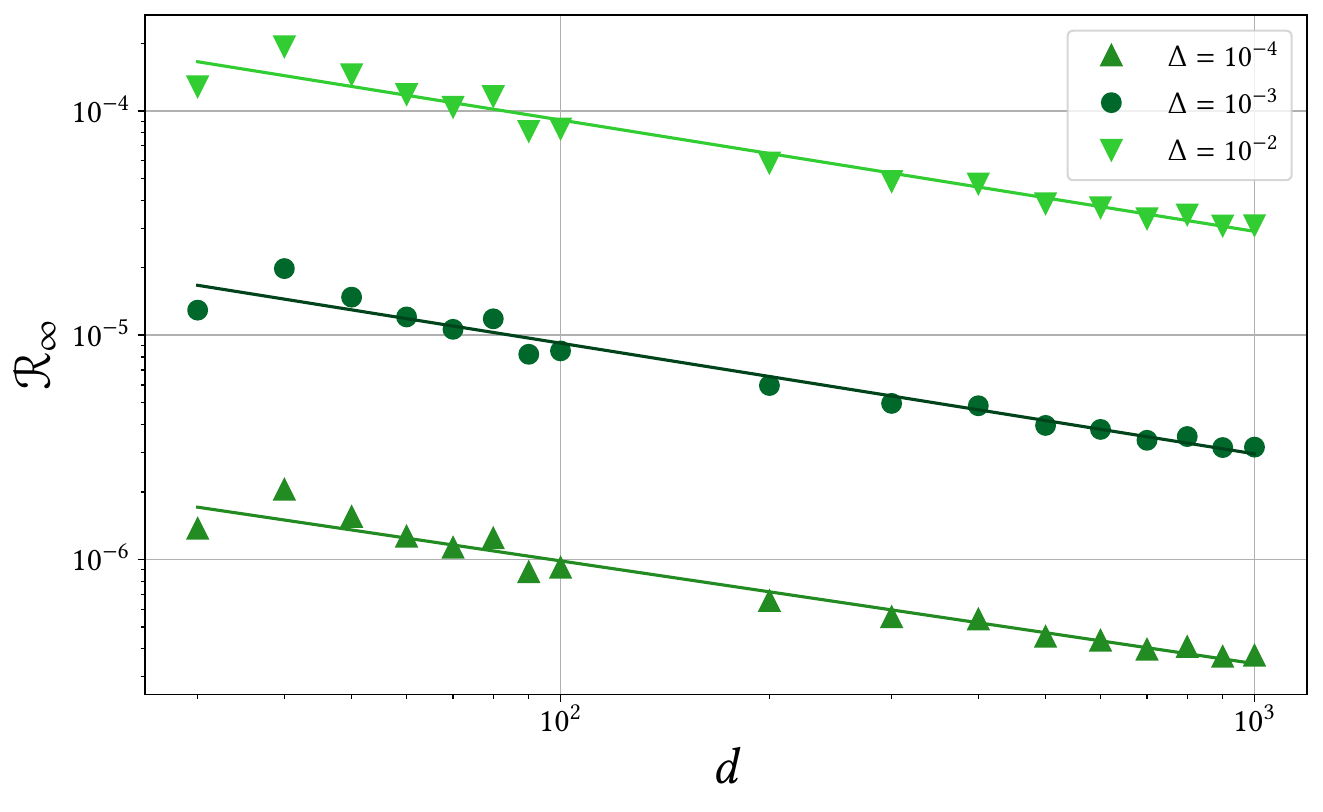}
  \vskip -0.05in
  \caption{Asymptotic population risk $\risk_\infty$ from simulations (dots) as a function of $\inp$ for different noise levels under the scaling $\kap=0$ and $\del=1/2$. The fitted straight lines have slopes $-0.458$, $-0.494$,  $-0.497$, for $\Delta = 10^{-4} , 10^{-3} , 10^{-2}$, respectively.}
  \label{fig:d32_finite_size}
\end{subfigure}
\vskip -0.05in
\caption{Network parameters:  $\hids_0 =8$, $\hidt=4$, $\prs = \delta_{rs}$. Activation function: $\act(x) = \erf(x/\sqrt{2})$. Data distribution: $\Prob (\x) =  \gauss(\x | \bm{0}, \Id )$ .}
\label{fig:deps_scal_eps050}
\end{figure*}

\begin{figure*}[tb!]
\centering
\begin{subfigure}[h]{.45\textwidth}
  \centering
  \includegraphics[width=\linewidth]{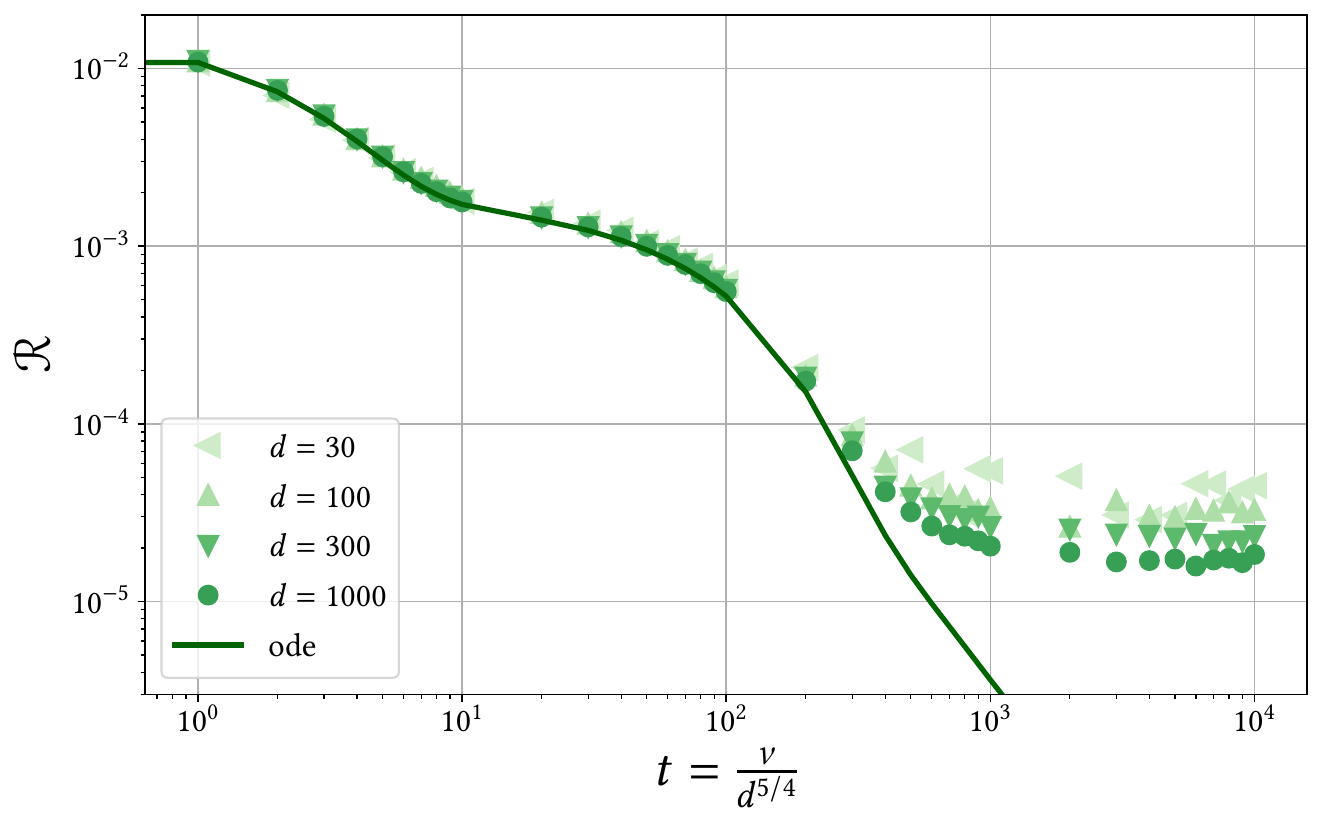}
  \vskip -0.05in
  \caption{Population risk dynamics for $\kap=0$ and $\del=1/4$. Fixed noise $\noise = 10^{-3}$ and varying $\inp$. Dots represent simulations, while the solid line is obtained by integration of the ODEs given by Eqs.~\eqref{eq:ode_kappa}. The data are compatible with the claim that as $\inp \to \infty$ the curve converges to zero population risk.}
  \label{fig:d32_vary_d_eps025}
\end{subfigure}%
\hspace{0.5cm}
\begin{subfigure}[h]{.45\textwidth}
  \centering
  \includegraphics[width=\linewidth]{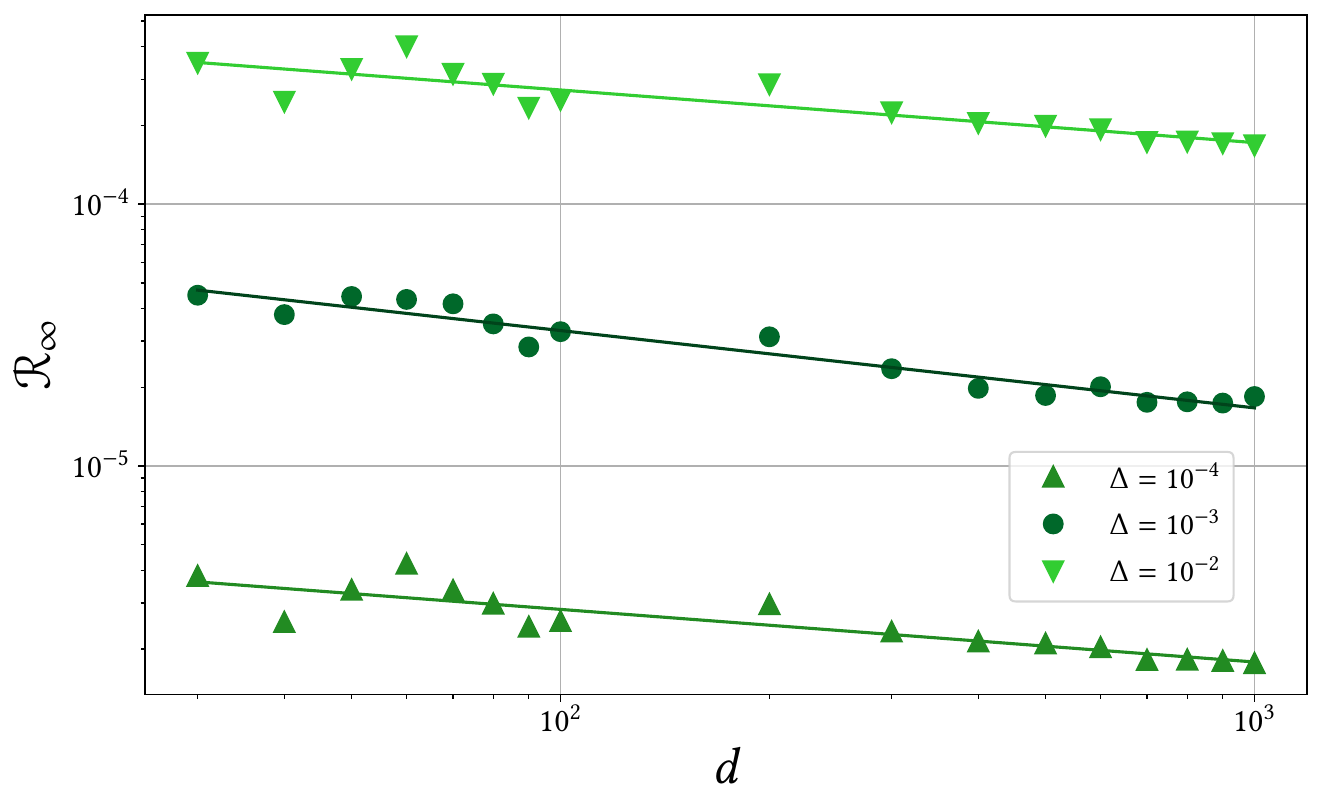}
  \vskip -0.05in
  \caption{Asymptotic population risk $\risk_\infty$ from simulations (dots) as a function of $\inp$ for different noise levels under the scaling $\kap=0$ and $\del=1/4$.The fitted straight lines have slopes $-0.201$, $-0.295$,  $-0.201$, for $\Delta = 10^{-4} , 10^{-3} , 10^{-2}$, respectively.}
 \label{fig:d32_finite_size_eps025}
\end{subfigure}
\vskip -0.05in
\caption{Network parameters: $\hids_0 =8$, $\hidt=4$, $\prs = \delta_{rs}$. Activation function: $\act(x) = \erf(x/\sqrt{2})$. Data distribution: $\Prob (\x) =  \gauss(\x | \bm{0}, \Id )$ .}
\label{fig:deps_scal_eps025}
\end{figure*}

As already stated, the interplay between the exponents directly affects the time scale. We end this subsection by graphically illustrating this fact through simulations. Setting the noise to $\noise = 10^{-3}$ we compare the cases $\del =  0, 1/4, 3/8, 1/2$ in Figure~\ref{fig:plot_d_d32_compare}. All simulations are rendered on the scale $ \delta \alp_0 = 1 / \inp  $ to illustrate the trade-off between asymptotic performance and training time. 


\subsection{Bad learning for \texorpdfstring{$\kap = 0$}{K=0}}
We now quickly discuss the uncommon case of $\lr$ growing with $\inp$ within the orange region. In Figure \ref{fig:plot_kap0_onlyI4} we compare simulations varying $\inp$ with the solution of the ODEs given by Eqs.~\eqref{eq:ode_delkappa}. Both lead to poor results compared to the green and blue regions. Moreover, this regime presents strong finite-size effects, making it harder to observe the asymptotic ODEs at small sizes. However, the trend as $\inp$ increases is very clear from the simulations. As discussed in Section \ref{sec:main}, the more the learning term is attenuated on the ODEs, the worse they describe the dynamics.


\begin{figure*}[tb!]
\centering
\begin{subfigure}[h]{.45\textwidth}
  \centering
  \includegraphics[width=\linewidth]{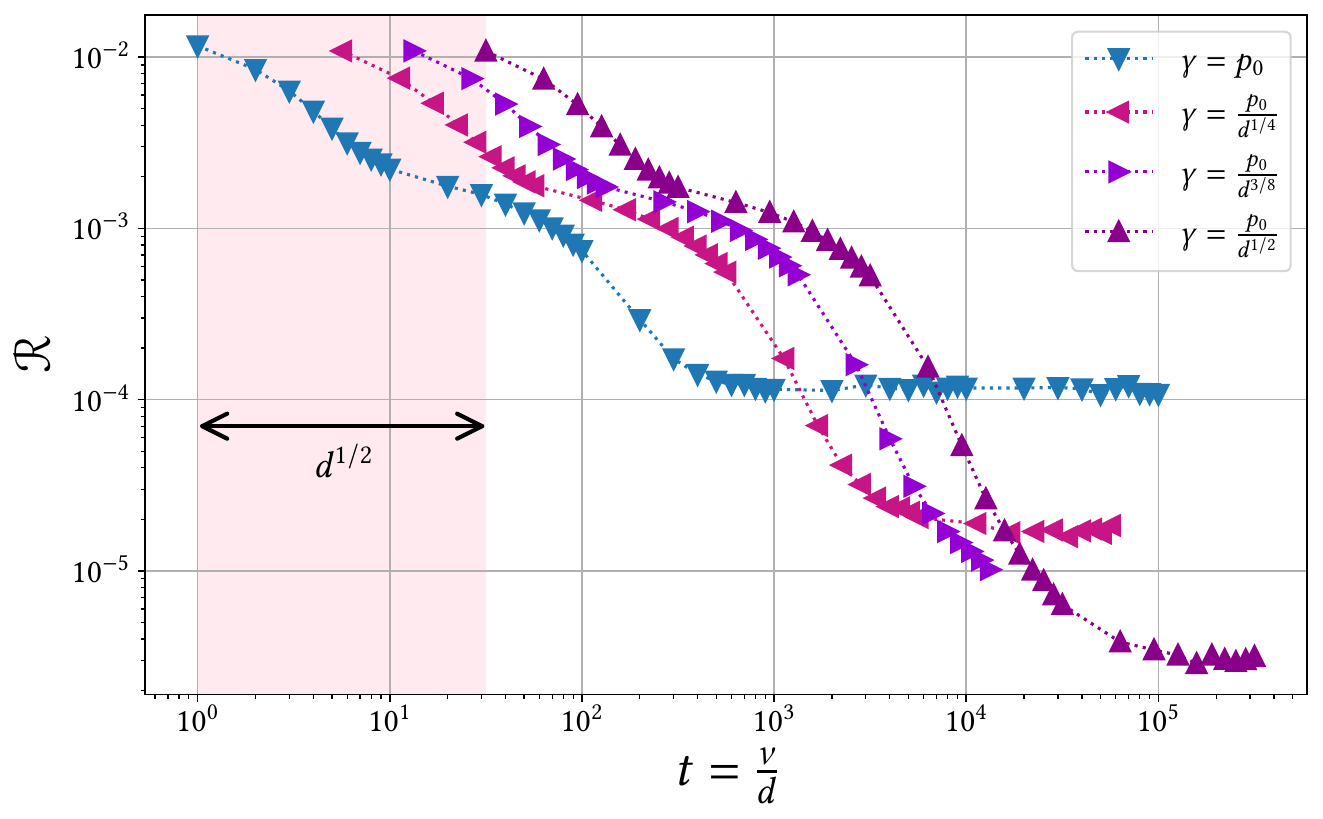}
  \vskip -0.05in    
  \caption{Simulations ($\inp = 1000$) for $\kap=0$ comparing different choices of the exponent $\del$. The final plateau is proportional to learning rate: $\risk_\infty \propto \lr \noise $.}
  \label{fig:plot_d_d32_compare}
\end{subfigure}%
\hspace{0.5cm}
\begin{subfigure}[h]{.45\textwidth}
  \centering
  \includegraphics[width=\linewidth]{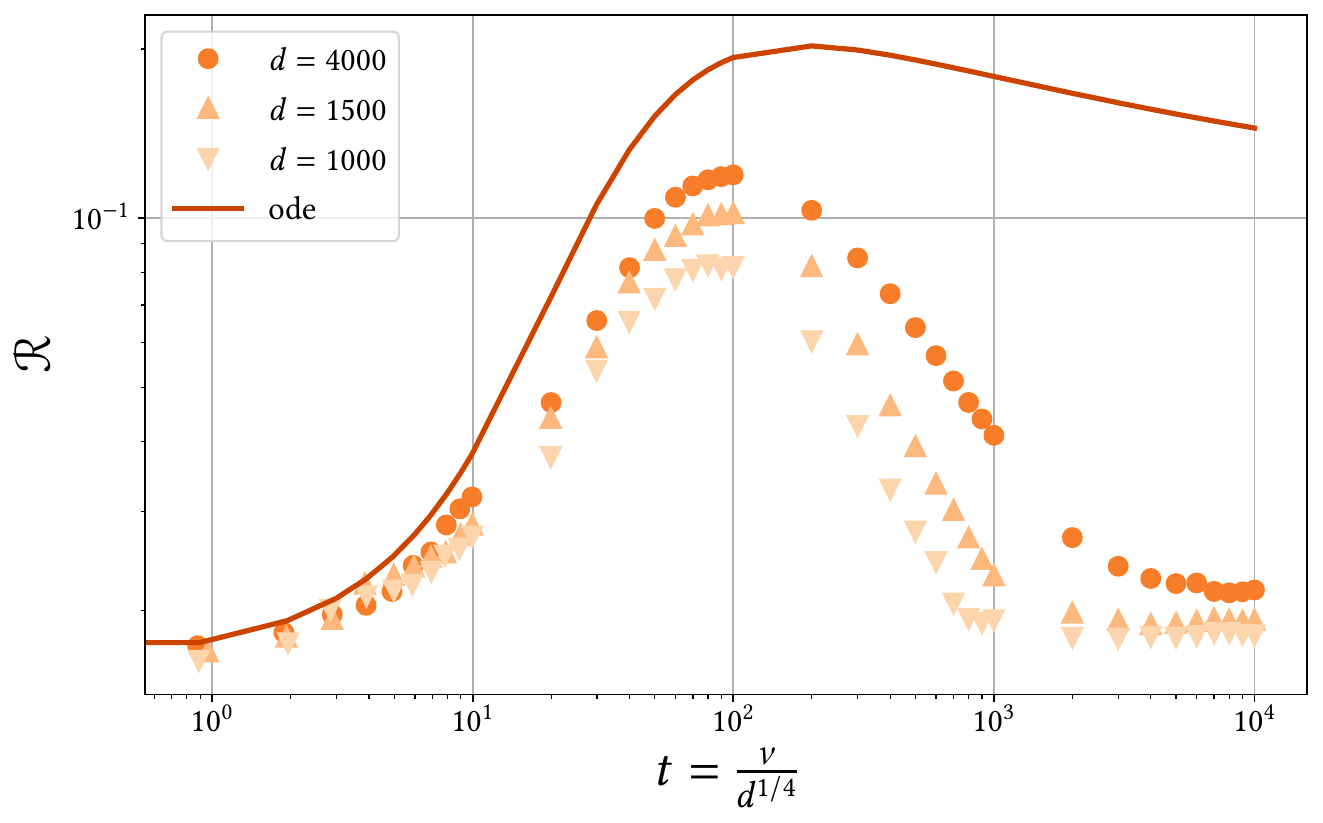}
  \vskip -0.05in
  \caption{Population risk dynamics for $\kap = 0$ and $\del = -3/8$. Dots represent simulations, while the solid line is obtained by integration of the ODEs given by Eqs.~\eqref{eq:ode_delkappa}.}
  \label{fig:plot_kap0_onlyI4}
\end{subfigure}
\vskip -0.05in
\caption{Network parameters $\hids_0 =8$, $\hidt=4$, $\prs = \delta_{rs}$. Noise level $\noise = 10^{-3}$. Activation: $\act(x) = \erf(x/\sqrt{2})$. Data distribution: $\Prob (\x) =  \gauss(\x | \bm{0}, \Id )$ .}
\end{figure*}

\subsection{Large hidden layer: \texorpdfstring{$\kap > 0$}{K > 0}}
Finishing our voyage through Figure \ref{fig:phase_diagram} with examples, we briefly discuss the case where both input and hidden layer widths are large. Although Theorem \ref{th:conv_eps} provides non-asymptotic guarantees for $\kap > 0$, the number of coupled ODEs grows quadratically with $\hids$, making the task of solving them rather challenging. Thus, we present simulations that illustrate the regions of Figure \ref{fig:phase_diagram}. Fixing $\inp=100$ we show in Figure \ref{fig:simul_phdiagr} learning curves for different values of $\kap$ and $\del$. The colors are chosen to match their respective regions in the phase diagram. 

Due to the relatively small sizes used in Figure \ref{fig:simul_phdiagr}, the green dots seem to decrease towards perfect learning, even when $\del < 0$, provided that $\kap$ is large enough, as is predicted by the phase diagram in Figure \ref{fig:phase_diagram}. Moreover, since $\inp$ is not large enough, when the parameters are within the orange region the finite-size effects actually dominates, similarly to Figure \ref{fig:plot_kap0_onlyI4}. The learning contribution still plays a role and the asymptotic population risk is similar to the case $\kap=\del=0$. Within the red region, which is out of scope of our theory, the simulation gets stuck on a plateau with larger population risk.

\begin{figure}[tb!]
\begin{center}
\centerline{\includegraphics[width=0.45\textwidth]{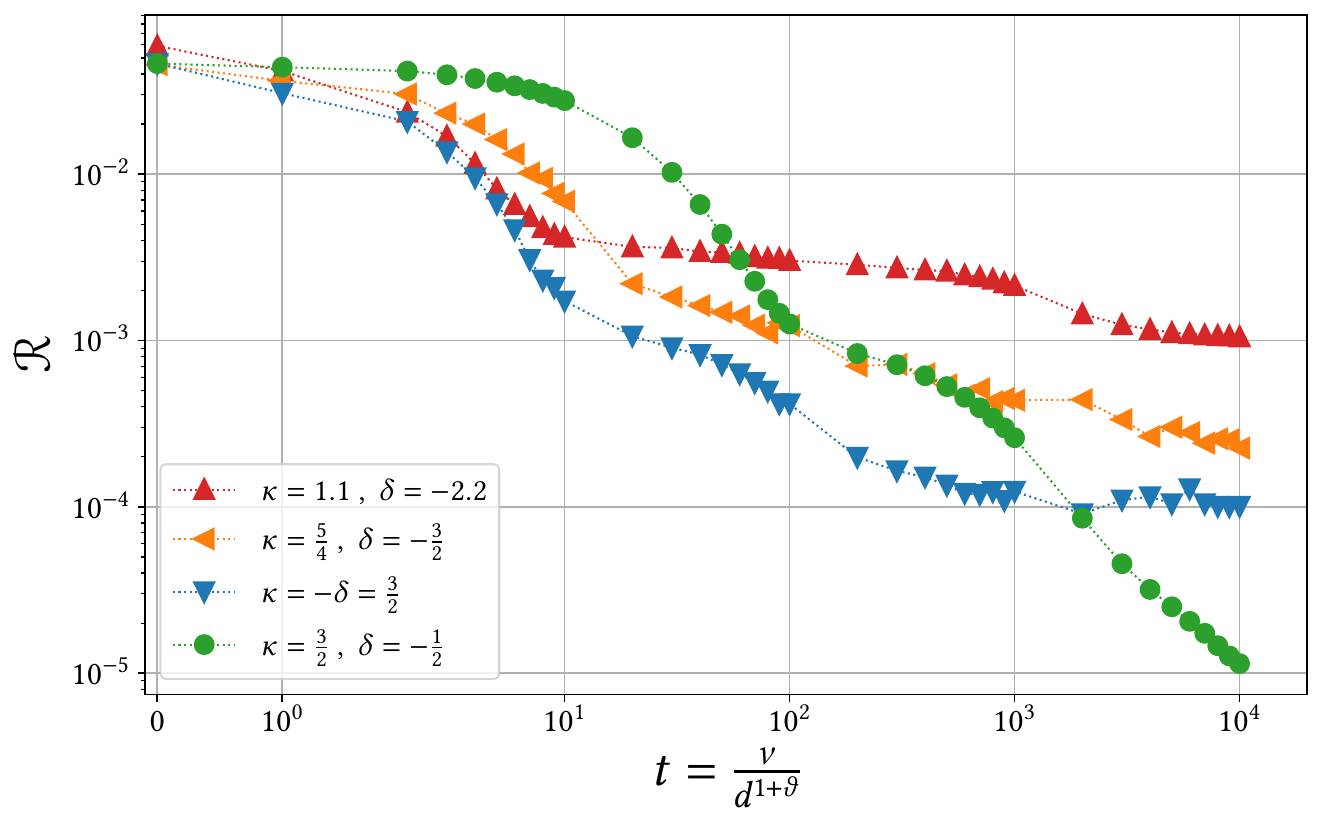}}
\vskip -0.05in
\caption{Simulations across different regions of Figure \ref{fig:phase_diagram}. Networks parameters $\inp=100$, $\hids=\inp^\kap$, $\lr= \inp^{-\del}$, $\hidt=4$, $\prs = \delta_{rs}$. Noise: $\noise = 10^{-3}$. Activation function: $\act(x) = \erf(x/\sqrt{2})$. Data distribution: $\Prob (\x) =  \gauss(\x | \bm{0}, \Id )$. Time scaling: $\vartheta = \kap+\delta$ for green and blue and $\vartheta = 2(\kap+\delta)$ for orange. The colors match Figure \ref{fig:phase_diagram}.}
\label{fig:simul_phdiagr}
\end{center}
\vskip -0.2in
\end{figure}
\section{Conclusion}
Building up on classical statistical physics approaches and extending them to a broad range of learning rate, time scales, and hidden layer width, we rendered a sharp characterisation of the performance of SGD for two-layer neural networks in high-dimensions. Our phase diagram describes the possible learning scenarios, characterizing learning regimes which had not been addressed by previous classical works using ODEs. Crucially, our key conclusions do not rely on an explicit solution, as our theory allows the characterization of the learning dynamics {\it without} solving the system of ODEs. The introduction of scaling factors is non-trivial and has deep implications. Our generalized description enlightens the trade-off between learning rate and hidden layer width, which has also been crucial in the mean-field theories.
\section*{Acknowledgements}
We thank G\'erard Ben Arous, Lena{\"i}c Chizat, Maria Refinetti and Sebastian Goldt for discussions. We acknowledge funding from the ERC under the European Union’s Horizon 2020 Research and Innovation Program Grant Agreement 714608- SMiLe. RV was partially financed by the Coordena\c{c}\~{a}o de Aperfeiçoamento de Pessoal de N\'{i}vel Superior - Brasil (CAPES) - Finance Code 001. RV is grateful to EPFL and IdePHICS lab for their generous hospitality during the realization of this project.

\newpage
\appendix
\section*{Appendix}\label{sec:app}
\addcontentsline{toc}{section}{\nameref{sec:app}}
\input{appendix/theorem}
\input{appendix/ode_perturbation}

\input{appendix/expectations}

\input{appendix/init_cond}

\newpage

\addcontentsline{toc}{section}{References}
\bibliographystyle{IEEEtran}
\typeout{}\bibliography{bibtex}

\end{document}

%% file: setting_customs.tex
\theoremstyle{plain}
\newtheorem{theorem}{Theorem}[section]

\newtheorem{lemma}[theorem]{Lemma}

\theoremstyle{definition}

\newtheorem{assump}{Assumption}[theorem]
\theoremstyle{remark}

\DeclareMathOperator{\E}{\mathbb{E}}
\DeclareMathOperator{\Prob}{\mathbb{P}}
\DeclareMathOperator{\R}{\mathbb{R}}
\DeclareMathOperator{\N}{\mathbb{N}}
\DeclareMathOperator{\Id}{\mathds{1}}
\DeclareMathOperator{\Er}{{\cal E}}
\DeclareMathOperator{\Or}{{\cal O}}

\def\inp{d}
\def\hids{p}
\def\hidt{k}
\def\nsamp{n}
\def\i{\nu}
\def\noise{\Delta}
\def\noisevar{\zeta}
\def\act{\sigma}
\def\lr{\gamma}
\def\w{\bm{w}}
\def\W{\bm{W}}
\def\x{\bm{x}}
\def\y{\bm{y}}
\def\lf{\lambda}
\def\alp{t}
\def\qjl{q_{jl}}
\def\q{q}
\def\Q{\bm{Q}}
\def\mjr{m_{jr}}
\def\m{m}
\def\M{\bm{M}}
\def\prs{\rho_{rs}}
\def\p{\rho}
\def\P{\bm{P}}
\def\Om{\bm{\Omega}}
\def\Omel{\Omega}

\def\gauss{{\cal N}}

\def\del{\delta}
\def\kap{\kappa}
\def\loss{{\cal L}}
\def\S{S}
\def\T{\tau}
\def\S{S}
\def\risk{{\cal R}}
\def\risks{{\cal R}_\text{s}}
\def\riskt{{\cal R}_\text{t}}
\def\riskst{{\cal R}_\text{st}}
\def\Explf{\E_{ \bm{\lf},\bm{\lf}^* \sim \gauss(\bm{\lf},\bm{\lf}^* |\bm{0},\Om)} }
\def\Explfs{\E_{ \bm{\lf} \sim \gauss(\bm{\lf} |\bm{0}, \Q)} }
\def\Explft{\E_{ \bm{\lf}^* \sim \gauss(\bm{\lf}^* |\bm{0},\P)} }

\def\Expnoise{\E_{\noisevar \sim \Prob (\noisevar)}}
\def\Expxy{\E_{\x, y \sim \Prob (\x, y)}}

\def\Expxgauss{\E_{\x\sim\gauss(\x|\bm{0},\Id )}}
\def\LambE{\Omel^{\alpha\beta\gamma}}
\def\LambEM{\bm{\Om}^{\alpha\beta\gamma}}
\def\LambEE{\Omel^{\alpha\beta\gamma\delta}}
\def\LambEEM{\bm{\Om}^{\alpha\beta\gamma\delta}}
\def\LambEEbar{\bar{\Omel}^{\alpha\beta\gamma\delta}}

\def\f{f}

\def\LambER{\bm{\Om}^{\alpha\beta}}
\def\LambERscal{\Omel^{\alpha\beta}}


%% file: appendix/theorem.tex
\section{Deterministic scaling limit of stochastic processes}\label{app:theorem}

In order to show the deterministic scaling of online SGD under a proper chosen time scale, we will make use of a convergence result by \cite{wang_2018, wang_2019}, which is adapted below in Theorem \ref{th:wang}.

\begin{theorem} [Deterministic scaling limit of stochastic processes] \label{th:wang} Consider a $\inp$-dimension discrete time stochastic process sequence, $   \{  \Om^\i  \; ; \; \i = 0, 1, 2, ..., [\S \T ] \}_{\S = 1, 2, ...}$ for some $ \T > 0$. The increment $\Om^{\i + 1} - \Om^\i $ is assumed to be decomposable into three parts, 
\begin{equation}
\label{eq:process_3parts}
    \Om^{\i + 1} - \Om^\i = \frac{1}{\S} \psi (  \Om^\i     ) + \bm{\Lambda}^\i + 
    \bm{\Gamma}^\i     \;,
\end{equation}
such that 
\begin{assump}
\label{A1}
The process $\tilde{\bm{\Lambda}}^\i \equiv \sum_{\nu' = 0}^\i \bm{\Lambda}^{\nu'}$ is a martingale and $\E \norm{ \bm{\Lambda}^\i }^{2}  \le C(\T)^2 / \S^{1 + \epsilon_1}$ for some $\epsilon_1 > 0$.
\end{assump}
\begin{assump}
\label{A2}
$ \E  \norm{ \bm{\Gamma}^\i  } \le C(\T) / \S^{1 + \epsilon_2}$ for some $\epsilon_2 > 0$.
\end{assump}
\begin{assump}
\label{A3}
 The function $\psi (  \Om )$ is Lipschitz, i.e, $ \norm{ \psi (  \Om ) - \psi (  \tilde{\Om} )  } \le C  \norm{ \Om-\tilde{\Om}}$ for any $\Om$ and $\tilde{\Om}$.
\end{assump}

Let $\Om(\alp)$, with $ 0 \le \alp \le \T $, be a continuous stochastic process such that $ \Om(\alp) = \bm{\Om}^\i$ with $\i = [\S\alp]$. Define the deterministic ODE
\begin{equation}
\label{eq:ODE_lemma}
   \dv{\alp}  \bar{\Om} (\alp) = \psi ( \bar{\Om}(\alp)   ) \;,
\end{equation}
with $  \bar{\bm{\Om}}(0)  =  \bar{\Om}_0  $.

Then, if assumptions \ref{A1} to \ref{A3} hold and assuming $  \E \norm{\Om^0 - \bar{\Om}_0 }  < C / \S^{\epsilon_3} $ for some $ \epsilon_3 > 0$ then we have for any finite $\S$:
\begin{equation}\label{eq:nonasymp_bound}
\E\norm*{\Om^\i -\bar{\Om} \left( \frac{\i}{\S} \right)  } \le C(\T) e^{c\T}   
     \S^{ -\min \{ \frac{1}{2} \epsilon_1 , \epsilon_2 , \epsilon_3 \}  }\;,
\end{equation}
where $\bar{\Om} (\cdot)$ is the solution of Eq.(\ref{eq:ODE_lemma}).

\begin{proof}
The reader interested in the proof is referred to the supplementary materials of \cite{wang_2018,wang_2019}.
\end{proof}

\end{theorem}

Although the theorem wasn't originally proven in the $p \to \infty$ setting, a glance at its proof shows that it still holds upon replacing $C(\tau)$ by $C(p, \tau)$ in Assumption \ref{A1} and \ref{A2}, as well as Equation \eqref{eq:nonasymp_bound}. We choose $\norm{\cdot}$ to be the $L^\infty$ norm, since it suits better the $p \to \infty$ scaling. The $S$ in Theorem \ref{th:wang} corresponds to $1 / \delta t$, where $\delta t$ is defined in Theorem \ref{th:conv_eps}.

Following \cite{wang_2018}, we define for $j, l \in [p]$
\[ \Psi_{jl}(\Om; \x) = \frac{\lr}{\hids\inp \, \delta \alp} \left( \Er_{j}^\i \lf_{l}^\i  + \Er_{l}^\i \lf_{j}^\i  \right)+ \frac{\lr^2 }{\hids^2\,\inp\, \delta \alp}\Er_{j}^\i \Er_{l}^\i, \]
and
\[ \psi_{jl}(\Om) = \Expxgauss\left[ \Psi_{jl}(\Om; \x ) \right]. \]
The functions $\Psi, \psi$ are similarly defined on $[p] \times [p+1, p+k]$. With that, we write
\[ \Om^{\i+1} - \Om^\i = \frac1S \psi(\Om) + \underbrace{\frac1S\left(\Psi(\Om^\i; \x ) - \psi(\Om^\i)\right)}_{\bm\Lambda^\i} + \bm\Gamma^\i, \]
where for $j, l \in [p]$
\[ \Gamma^\i_{jl} = \frac{\gamma^2}{p^2d^2} \left( \lVert \x \rVert_2^2 - d\right) \Er_j^\i\Er_l^\i. \]

The main obstacle to bounding $\bm\Lambda^\i$ and $\bm \Gamma^\i$ is the fact that the $q_{jj}$ can a priori diverge to infinity. Our first task is therefore to show that this does not happen; as a proxy we show a subgaussian-like moment bound:
\[ \E\left[(q_{jj}^\i)^t\right] \leq \left(C(\T) + \frac{c t}{S}\right)^t. \]

Equipped with the above bound, controlling 
$ \E \norm{ \bm{\Lambda}^\i }^{2} $
and
$\E  \norm{ \bm{\Gamma}^\i  } $ 
becomes fairly easy. All proof details are in the below sections.

\subsection{Preliminaries: bounding the \texorpdfstring{$q_{jj}$}{qjj}}

Since $\act$ is $L$-Lipschitz, we have by the Cauchy-Schwarz inequality
\begin{equation}\label{eq:app:bound_error}
(\Er^\i)^2 \leq \frac {3L^2} \hidt \sum_{r=1}^\hidt (\lf_r^*)^2 + \frac {3L^2} \hids \sum_{j=1}^\hids (\lf_j)^2 + 3\noise \noisevar^2 \equiv \Phi^\i
\end{equation}
Define
\[ s^\i = \E\Phi ^\i =  \frac {3L^2} \hidt \sum_{r=1}^\hidt \p_{rr} + \frac {3L^2} \hids \sum_{j=1}^\hids \q_{jj}^\i + 3\noise\]
Assumption \ref{ass_main:2} in Theorem \ref{th:conv_eps} implies that
\[ |\q_{jj}^{\i+1} - \q_{jj}^\i| \leq \frac 1 S \left(c_1 (\lambda_j^\i)^2 + c_2 (\Er^\i)^2 \right) \]
where $c_1, c_2$ are absolute constants. Summing those inequalities yield
\[ |s_{\i+1} - s^{\i}| \leq \frac{c_3}S \Phi^\i,  \]
and finally
\[ \E_\i[s^{\i+1}] \leq s^\i \left(1 + \frac{c_3}{S}\right) \leq s^\i e^{c_3/S}. \]
As a result, we have for any $0 \leq \i \leq S \tau$
\begin{equation} 
\E[s^\i] \leq c_4 e^{c_3 \tau}.
\end{equation}
For simplicity, let $q^\nu$ denote any of the $q_{jj}^\nu$. We have, for all $t \geq 0$, 
\[ (\q^{\i+1})^t - (\q^{\i})^t = t (\q^{\i})^{t-1} (\q^{\i+1} - q^\i) + O\left(\frac{t^2}{S^2}\right), \]
where the remainder term has bounded expectation. Again, we write
\[ \left|(\q^{\i+1})^t - (\q^{\i})^t \right| \leq t (\q^{\i})^{t-1} \frac1S(c_1 (\Er^\i)^2 + c_2 (\lf^\i_i)^2) + \frac{c_5 t^2}{S^2}. \]
By Assumption \ref{ass_main:3}, the $q_{ii}^\i$ are bounded from below by a constant, hence
\[ \E_\i[(\q^{\i+1})^t] \leq  (\q^{\i})^t \left(1 + \frac{c_6 t}{S} \right) + O\left(\frac{c_5 t^2}{S^2}\right)  \]
This implies that for any $t \geq 0$ and $0 \leq \i \leq S\T$,
\begin{equation}\label{eq:bound_q_powers}
    \E[(\q^{\i})^t] \leq \left(c_7 + \frac{c_5t^2}{S}\right)e^{c_6 \T} \leq \left(C(\T) + \frac{c_5 t}{S}\right)^t
\end{equation}

\subsection{Assumption \ref{A1}}

We have for all $i, j \in [p+k]$,
\begin{equation}
    \nonumber
   \left(\Omega_{ij}^{\i+1} - \E_\i [\Omega_{ij}^{\i+1}]\right)^2 \leq 2\left((\Omega_{ij}^{\i+1} - \Omega_{ij}^{\i})^2 + (\Omega_{ij}^{\i} - \E_\i[\Omega_{ij}^{\i+1}])^2\right) \;.
\end{equation}
As a consequence,
\begin{equation}
    \nonumber
  \E \norm{ \bm{\Lambda}^\i }^{2} \leq 4 \max_{i,j} (\Omega_{ij}^{\i+1} - \Omega_{ij}^{\i})^2 \;.
\end{equation}

Now, by definition,
\[ (\q_{ij}^{\i+1} - \q_{ij}^\i)^2 \leq \frac {L} {S^2} \left(c_1(\Er^\i)^2 + c_2|\Er^\i|(|\lf_i| + |\lf_j|) \right)^2 \leq \frac L {S^2} \left( c_3 (\Er^\i)^4 + c_4 (\max_\ell \lf^\i_\ell)^4 \right), \]
The term in $(\Er^\i)^4$ is bounded by the same techniques as the last section. For the second term,
\[ \E_\i\left[(\max_\ell \lf_\ell)^4 \right] \leq c_5  \log(p)^2 \left(\max_\ell q^\i_{\ell\ell} \right)^4, \]
and we can write for any $t \geq 0$
\[ \max_\ell{(q_{\ell \ell}^\i)^4} \leq \left( \sum_{\ell} (q_{\ell\ell}^\i)^t \right)^{4/t}. \]
By Jensen's inequality, for $t \geq 4$
\[ \E\left[\left(\max_\ell{q_{\ell \ell}^\i}\right)^4 \right] \leq \left( \sum_{\ell} \E[(q_{\ell\ell}^\i)^t \right)^{4/t} \leq p^{4/t} \left(C(\T) +\frac{c_6 t}{S} \right)^4, \]
using \eqref{eq:bound_q_powers}. Choosing $t = 4 \log(p) \ll S$ shows that
\[ \E\left[ \max_{i,j} (\q_{ij}^{\i+1} - \q_{ij}^\i)^2 \right] \leq \frac{C(\T) \log(p)^2}{S^2}  \]

A similar bound holds for the $\m_{ij}$, and hence
\begin{equation}
    \nonumber
     \E \norm{ \bm{\Lambda}^\i }^{2}  \leq \frac{c_5 \log(p)^2}{S^2} \;,
\end{equation}
which implies Assumption \ref{A1} with $\epsilon_1 = 1$ and $C(p, \T) = C'(\T)\log(p)$.

\subsection{Assumption \ref{A2}}
Since $\act$ is Lipschitz, for any $i, j \in [p]$
\[ \Er^\i_i \Er^\i_j \leq L^2 (\Er^\i)^2. \]
Hence,
\begin{align*} 
\E[\lVert\bm\Gamma^\i\rVert_\infty] &\leq \frac{L^2 \lr^2}{\inp^2\hids^2}\E\left[\left(\lVert \x \rVert_2^2 - d \right) \Phi^\i \right]\\ 
&\leq \frac{L^2 \lr^2}{\inp^2\hids^2} \left(\frac1{2\sqrt{d}} \E\left[ \left(\lVert \x \rVert_2^2 - d \right)^2 \right] + \frac{\sqrt{d}}2 \E\left[ (\Er^\i)^4 \right]\right).
\end{align*}
The first expectation is the variance of a $\chi^2_d$ random variable, which is equal to $2d$, and the second expectation is bounded by the same methods as the above sections. The term in brackets is therefore bounded by $c_1\sqrt{d}$, and
\[ \E[\lVert\bm\Gamma^\i\rVert_\infty] \leq c_2  \frac{\lr^2}{\inp^{3/2}\hids^2} \]
Finally, since for any $y > 0$ we have $y^2 \leq \max(y, y^2)^{3/2}$, letting $y = \lr/\hids$ we find
\[ \E[\lVert\bm\Gamma^\i\rVert_\infty] \leq c_2 \max\left( \frac{\lr}{\hids \inp}, \frac{\lr^2}{\hids^2 \inp}  \right)^{3/2} \leq c_3 (\delta t)^{3/2}, \]
hence Assumption \ref{A2} is true with $\epsilon_2 = 1/2$.

\subsection{$\surd$-Lipschitz property}

Let $\Om, \Om' \in \R^{(\hids+\hidt)\times(\hids+\hidt)}$, we can write the $(i, j)$ coefficient of $\psi(\Om)$ as $f_{ij}(\sqrt{\Om}) $, where 
\begin{align*} 
    f:  \R^{(\hids+\hidt)\times(\hids+\hidt)} &\to \R\\
        A &\mapsto \E_{x \sim \mathcal \gauss(0, I_{p+k})} [g_{ij}(Ax)]
\end{align*}
The same arguments as above show that the function $f$ is Lipschitz, and hence for some constant $L''$ we have
\[ \lVert\psi(\Om) - \psi(\Om')\rVert \leq L'' \lVert \sqrt{\Om} - \sqrt{\Om'} \rVert. \]

%% file: appendix/ode_perturbation.tex
\section{A lemma on ODE perturbation}

In this section, we prove a proposition that bounds the difference between an ODE solution and a perturbed version, for a bounded time $t$.

\begin{theorem}\label{th:ode_perturbation}
    Let $f, g: \mathbb R^n \to \mathbb R^n$ be two $L$-Lipschitz functions, and consider the following differential equations in $\mathbb R^n$:
    \begin{align*}
       \dv{\x}{t}  &= f( \x ) + \epsilon g( \x ),\\
        \dv{\bm{y}}{t}  &= f( \bm{y} ),
    \end{align*}
    where $\epsilon > 0$, and with the initial condition 
    $ \x (0) = \bm{y} (0) $ . Then, if $\T > 0$ is fixed, we have
     \[ \lVert \x (t) - \bm{y} (t) \rVert_2 \leq c \epsilon e^{L \tau}  \]
    for any $0 \leq t \leq \T$, with $c$ a constant independent from $\epsilon, \T$.
\end{theorem}

Before proving this proposition, we begin with a small lemma:

\begin{lemma}\label{lem:norm_ode_solution}
    Let $a, b > 0$, and $z: \mathbb R^+ \to \mathbb R^+$ a function satisfying
    \[ \dv{z}{t}  = a z + b \sqrt{z} \]
    with $z(0) = 0$. Then, for some constant $c > 0$, we have
    \[ z(t) \leq c \frac{b^2 e^{at}}{a^2} \quad \text{for all} \quad t\geq 0\]
\end{lemma}
\begin{proof}
    Upon considering the function $a^2 z(t / a) / b^2$ instead, we can assume that $a = b = 1$. Then, we have
    \[ \dv{z}{t}   \leq \max(z, 1) + \max(\sqrt{z}, 1), \]
    and the RHS is an increasing function. Hence, if $\tilde z$ is a solution of
    \[ \dv{\tilde z}{t}  = \max(z, 1) + \max(\sqrt{\tilde z}, 1), \]
    with $\tilde z(0) = 0$, then $z(t) \leq \tilde z(t)$ for all $t \geq 0$. Since the RHS of the above equation is Lipschitz everywhere, we can apply the Picard–Lindelöf theorem, and check that the unique solution to this equation is
    \[ \tilde z(t) = \begin{cases}
        2 t & \text{if } t\leq \frac{1}{2}\\
        (c_1 e^{t} -c_2)^2 & \text{otherwise}
    \end{cases},\]
    where $c_1$ and $c_2$ are ad hoc constants. The lemma then follows from adjusting the constant $c$ as needed.
\end{proof}

We are now in a position to show Theorem \ref{th:ode_perturbation}:
\begin{proof}
    Assume for simplicity that 
    $\x(0) = \bm{y}(0) = \bm{0}$. We begin by bounding 
    $\x (t) $; we have
      \begin{equation}
        \nonumber
        \dv{\lVert \x \rVert^2}{t}  = 2  \x^\top  \dv{\x}{t}   \leq 2  \lVert \x \rVert  \; \lVert f( \x ) + \epsilon g( \x ) \rVert \;.
    \end{equation}
    By the Lipschitz condition,
     \begin{equation}
    \nonumber
        \lVert f( \x ) + \epsilon g( \x ) \rVert \leq \lVert f( \bm{0} ) + \epsilon g( \bm{0} ) \rVert + \frac{L}{2} \lVert \x \rVert \;,
    \end{equation}
    so that
    \begin{equation}
    \nonumber
       \dv{\lVert \x \rVert^2}{t} \leq L \lVert \x \rVert^2 + 2 \lVert f ( \bm{0} ) + \epsilon g( \bm{0} ) \rVert \; \lVert \x \rVert   \;.
    \end{equation}
    Applying Lemma \ref{lem:norm_ode_solution} and taking square roots on each side,
       \begin{equation}
    \label{eq:bound_norm_x}
        \lVert \x (t) \rVert \leq c\frac{\lVert f( \bm{0} ) + \epsilon g( \bm{0} ) \rVert }{L} e^{Lt/2} \leq c\frac{\lVert f( \bm{0} ) + \epsilon g( \bm{0} ) \rVert }{L} e^{L\T/2} \;,
    \end{equation}
    for any $0 \leq t \leq \T$. Now, similarly,
     \begin{align*}
     \dv{\lVert \x - \y \rVert^2}{t}  & \leq 2  \lVert \x - \y \rVert \left\lVert \dv{( \x - \y )}{t}  \right\rVert\\
    &\leq 2 \lVert \x - \y \rVert \; \lVert f( \x ) - f( \y ) + \epsilon g( \x ) \rVert \\
    &\leq L  \lVert \x - \y \rVert^2 + 2 \epsilon \lVert g( \x ) \rVert \; \lVert \x - \y \rVert \\
    &\leq L \lVert \x - \y \rVert^2 +  \epsilon \left( \lVert g( \bm{0} ) \rVert + c\lVert f( \bm{0} ) + \epsilon g( \bm{0} ) \rVert e^{L\T/2} \right) \lVert \x - \y \rVert \;,
    \end{align*}
    having used \eqref{eq:bound_norm_x} on the last line.
    This is again the setting of Lemma \ref{lem:norm_ode_solution}, which gives
      \begin{equation}
        \nonumber
        \lVert \x - \y \rVert \leq c_1 \epsilon e^{L\T/2} \frac{e^{Lt/2}}{L} \leq c_2 \epsilon e^{L\tau} \;.
    \end{equation}
\end{proof}

%% file: appendix/expectations.tex
\section{Expectations over the local fields}\label{app:expec}

In this appendix we present the explicit expressions from the expectations of the local fields used to compute the population risk and the ODE terms.

\subsection{Population risk}

We write the population risk~\eqref{eq:pop_risk} as
\begin{equation}
\begin{split}
      \risk ( \Om  )  &= \Explf \Expnoise 
    \left[ \left( \hat{\f}( \bm{\lf} ) - \f ( \bm{\lf}^* )  \right)^{2}   \right]   \\
    &=  \riskt ( \P  ) + \risks ( \Q  ) + \riskst (\P , \Q, \M )   \;,
\end{split}
\end{equation}
with
\begin{subequations}
\label{eq:pop_risk_3}
\begin{equation}
    \riskt \equiv \Explft \left[ \f ( \bm{\lf}^* )^2   \right] =\frac{1}{\hidt^2} \sum_{r, s =1}^\hidt \Explft\left[ \act(\lf_{r}^*) \act(\lf_{s}^*)  \right] 
\end{equation}
\begin{equation}
    \risks \equiv \Explfs \left[\hat{\f}( \bm{\lf} )^2   \right] = 
    \frac{1}{\hids^2} \sum_{j, l =1}^\hidt \Explfs\left[ \act(\lf_{j}) \act(\lf_{l}) \right] \;,
\end{equation}
\begin{equation}
    \riskst \equiv \Explf \left[ \hat{\f}( \bm{\lf} ) \f ( \bm{\lf}^* )   \right] =  - \frac{2}{\hids\hidt}  \sum_{j=1}^\hids \sum_{r=1}^\hidt \Explf\left[ \act(\lf_{j}) \act(\lf_{r}^*) \right]
\end{equation}
\end{subequations}

Define the vector $ \bm{\lf}^{\alpha\beta} \equiv \left(  \lf^\alpha , \lf^\beta  \right)^\top \in \R^2  $, where the upper indices on the components indicate they may refer to student or teacher local fields. Consider the covariance matrix on the subspace spanned by $ \bm{\lf}^{\alpha\beta}$:
\begin{equation}
\label{eq:covE2}
\LambER \equiv \Explf \left[  \bm{\lf}^{\alpha\beta}  \left(\bm{\lf}^{\alpha\beta} \right)^\top  \right] \in \R^{2 \times 2} \;.
\end{equation}

For $\act(x) = \erf ( x / \sqrt{2} ) $ the expectations in Eqs.~\eqref{eq:pop_risk_3} are in general given by \cite{saad_1995}
\begin{equation}
    \Explf  \left[\act (\lf^\alpha ) \act( \lf^\beta )  \right] = 
    \frac{1}{\pi} \arcsin \left( \frac{   \LambERscal_{12}  }{\sqrt{\left( 1+ \LambERscal_{11}\right) \left( 1+ \LambERscal_{22} \right)}}    \right)   \;.
\end{equation}
where $   \LambERscal_{jl} \equiv (\LambER)_{jl} $ is an element of the covariance matrix given by Eq.~\eqref{eq:covE2}. 

Explicitly, the population risk contributions are
\begin{subequations}
\label{eq:pop_risk_3_exp}
\begin{equation}
    \riskt ( \P  ) =\frac{1}{\hidt^2} \sum_{r, s =1}^\hidt  \frac{1}{\pi} \arcsin \left( \frac{   \p_{r s}  }{\sqrt{\left( 1+ \p_{r r} \right) \left( 1+ \p_{s s} \right)}}    \right) \;,
\end{equation}
\begin{equation}
    \risks ( \Q  ) =    \frac{1}{\hids^2} \sum_{j,l  =1}^\hidt  \frac{1}{\pi} \arcsin \left( \frac{   \q_{j l}  }{\sqrt{\left( 1+ \q_{j j} \right) \left( 1+ \q_{ll} \right)}}    \right)  \;,
\end{equation}
\begin{equation}
    \riskst (\P,  \Q , \M )  =  - \frac{2}{\hids\hidt}  \sum_{j=1}^\hids \sum_{r=1}^\hidt  \frac{1}{\pi} \arcsin \left( \frac{   \m_{j r}  }{\sqrt{\left( 1+ \q_{j j} \right) \left( 1+ \p_{rr} \right)}}    \right)  \;.
\end{equation}
\end{subequations}

\subsection{ODE contributions}

From the update equations, we first consider the expectations linear in $\Er_j$:
\begin{subequations}
\label{eq:expE}
\begin{equation}
\begin{split}
  \Explf \Expnoise \left[  \Er_{j} \lf_{l}  \right]  =&  
  \frac{1}{\hidt} \sum_{r' =1}^{\hidt} \Explf  \left[\act' (\lf_{j}) \lf_l \act( \lf_{r'}^{*}  )  \right]
   \\
   &  -  \frac{1}{\hids} \sum_{l' = 1}^{\hids} \Explf \left[ \act' (\lf_{j}) \lf_l   \act ( \lf_{l'}  ) \right] \;, 
\end{split}
\end{equation}
\begin{equation}
\begin{split}
   \Explf \Expnoise   \left[  \Er_{j} \lf_{r}^*  \right]  =&    \frac{1}{\hidt} \sum_{r'=1}^{\hidt}   \Explf   \left[\act' (\lf_{j}) \lf_{r}^* \act( \lf_{r'}^{*}  )  \right]
   \\
   &   -   \frac{1}{\hids} \sum_{l' = 1}^{\hids} \Explf \left[ \act' (\lf_{j}) \lf_{r}^*   \act ( \lf_{l'}  ) \right]  \;.
\end{split}
\end{equation}
\end{subequations}
Define the vector $ \bm{\lf}^{\alpha\beta\gamma} \equiv \left(  \lf^\alpha , \lf^\beta , \lf^\gamma  \right)^\top \in \R^3  $, where the upper indices on the components indicate they may refer to student or teacher local fields. Consider the covariance matrix on the subspace spanned by $ \bm{\lf}^{\alpha\beta\gamma}$:
\begin{equation}
\label{eq:covE}
\LambEM \equiv \Explf \left[  \bm{\lf}^{\alpha\beta\gamma}  \left(\bm{\lf}^{\alpha\beta\gamma} \right)^\top  \right] \in \R^{3\times3} \;.
\end{equation}


For $\act(x) = \erf ( x / \sqrt{2} ) $ the expectations in Eqs.~\eqref{eq:expE} are given by \cite{saad_1995}
\begin{equation}
    \Explf  \left[\act' (\lf^\alpha ) \lf^\beta \act( \lf^\gamma )  \right] = 
    \frac{2}{\pi} \frac{ \LambE_{23} \left( 1+ \LambE_{11}\right) - \LambE_{12}\LambE_{13} }{ \left( 1+\LambE_{11} \right) \sqrt{\left( 1+\LambE_{11} \right) \left( 1+\LambE_{33} \right) - \left( \LambE_{13} \right)^2   }} \;,
\end{equation}
where $  \LambE_{jl} \equiv (\LambEM)_{jl} $ is an element of the covariance matrix given by Eq.~\eqref{eq:covE}. As examples, we write explicitly:
\begin{equation}
\bm{\Om}^{jl r'} = \begin{bmatrix}
\q_{jj} & \q_{jl} & \m_{j r'} \\
\q_{jl} & \q_{ll} & \m_{l r'} \\
\m_{j r'} & \m_{l r'} & \p_{r' r'}
\end{bmatrix} \; \; \; , \; \; \; 
\bm{\Om}^{j r r'} = \begin{bmatrix}
\q_{jj} & \m_{jr} & \m_{j r'} \\
\m_{jr} & \p_{rr} & \p_{r r'} \\
\m_{j r'} & \p_{r r'} & \p_{r' r'}
\end{bmatrix}  \;.
\end{equation}

The quadratic contribution in $\Er_j$ is given by
\begin{equation}
\label{eq:expEE}
\begin{split}
     \Explf \Expnoise \left[  \Er_{j} \Er_{l}  \right]  =&  
     \frac{1}{\hidt^2} \sum_{ r ,  r' =1}^{\hidt} \Explf  \left[\act' (\lf_{j}) 
     \act' (\lf_{l}) \act (\lf_{r}^*  ) \act ( \lf_{r'}^{*}  )  \right]  \\
     & + \frac{1}{\hids^2} \sum_{ j' ,  l' =1}^{\hids} \Explf  \left[\act' (\lf_{j}) 
     \act' (\lf_{l}) \act (\lf_{j'}  ) \act ( \lf_{l'}  )  \right]  \\ 
     & - \frac{2}{\hids\hidt} \sum_{ l' =1}^{\hids} \sum_{ r =1}^{\hidt} \Explf  \left[\act' (\lf_{j}) 
     \act' (\lf_{l}) \act (\lf_{r}^*  ) \act ( \lf_{l'}  )  \right]  \\
     & + \noise   \Explf  \left[\act' (\lf_{j}) 
     \act' (\lf_{l})  \right]
\end{split}
\end{equation}

The solution of the noise-dependent term can be constructed with the covariance matrix \eqref{eq:covE2} and is given by \cite{goldt_2019}
\begin{equation}
    \Explf  \left[\act' (\lf^\alpha)  \act' (\lf^\beta)  \right] =  \frac{2}{\pi} 
    \frac{1}{\sqrt{ 1+ \LambERscal_{11} + \LambERscal_{22} +
    \LambERscal_{11}\LambERscal_{22} - \left( \LambERscal_{12} \right)^2   }}    
\end{equation}

Similarly, one can define the vector $ \bm{\lf}^{\alpha\beta\gamma\delta} \equiv \left(  \lf^\alpha , \lf^\beta , \lf^\gamma , \lf^\delta  \right)^\top \in \R^4  $ and write the covariance matrix on the subspace spanned by $ \bm{\lf}^{\alpha\beta\gamma\delta}$:
\begin{equation}
    \LambEEM  \equiv \Explf \left[  \bm{\lf}^{\alpha\beta\gamma\delta}  \left( \bm{\lf}^{\alpha\beta\gamma\delta} \right)^\top \right]  \in \R^{4\times 4} \;.
\end{equation}

For $\act(x) = \erf ( x / \sqrt{2} ) $ the expectations in Eqs.~\eqref{eq:expEE} are given by \cite{saad_1995}
\begin{equation}
    \Explf  \left[\act' (\lf^\alpha ) \act' (\lf^\beta )  \act (\lf^\gamma ) \act( \lf^\delta  )  \right] = 
    \frac{4}{\pi^2} \frac{1}{ \sqrt{\LambEEbar_0} } \arcsin \left( \frac{ \LambEEbar_1 }{\sqrt{\LambEEbar_2 \LambEEbar_3}}   \right) \;,
\end{equation}
with
\begin{subequations}
\begin{equation}
    \LambEEbar_0 \equiv \left( 1+\LambEE_{11} \right) \left( 1+\LambEE_{22} \right) - \left( \LambEE_{12} \right)^2 \;,
\end{equation}
\begin{equation}
\begin{split}
   \LambEEbar_1 \equiv &  \LambEEbar_0 \LambEE_{34} - \LambEE_{23} \LambEE_{24} \left( 1+\LambEE_{11} \right) - \LambEE_{13} \LambEE_{14} \left(1 + \LambEE_{22} \right) \\
   & + \LambEE_{12} \LambEE_{13} \LambEE_{24} + \LambEE_{12}\LambEE_{14}\LambEE_{23} \;,
\end{split}
\end{equation}
\begin{equation}
\begin{split}
   \LambEEbar_2 \equiv & \LambEEbar_0 \left( 1 + \LambEE_{44} \right)  - \left(\LambEE_{24}\right)^2 \left(1+ \LambEE_{11} \right)   - \left(\LambEE_{13}\right)^2 \left(1+ \LambEE_{22} \right) \\
   & + 2
   \LambEE_{12} \LambEE_{13} \LambEE_{23} \;, \;.    
\end{split}
\end{equation}
\begin{equation}
\begin{split}
  \LambEEbar_3 \equiv & \LambEEbar_0 \left( 1 + \LambEE_{44} \right)  - \left(\LambEE_{24}\right)^2 \left(1+ \LambEE_{11} \right)   - \left(\LambEE_{14}\right)^2 \left(1+ \LambEE_{22} \right)  \\
  & + 2    \LambEE_{12} \LambEE_{14} \LambEE_{24} \;.    
\end{split}
\end{equation}
\end{subequations}

\subsection{From gradient flow to local fields}
\label{app:c:flow}
Consider the gradient flow approximation
\begin{align*}
    \dv{\w_j}{t}  &= -\nabla_{\w_j}\risk(\W, \W^*)\\
    &= - \frac1{p\sqrt{d}}\Expxgauss\left[ \x \act'(\lf_j ) \Er \right].
\end{align*}

Now, since for any $\bm{x}^{\top}\bm{y}$, we have
\begin{equation}
    \nonumber
    \dv{\left(\bm{x}^{\top}\bm{y}\right)}{t}  =  \bm{x}^{\top} \dv{\bm{y}}{t} + \bm{y}^{\top} \dv{\x}{t} \;,
\end{equation}
we find
\begin{equation}
    \nonumber
\dv{\q_{jl}}{t}   = -\frac{1}{pd} \Expxgauss\left[ \left(\act'(\lf_j)  \lf_l + \act'(\lf_l) \lf_j \right) \Er \right] \;.
\end{equation}
Recalling the definition $\Er_j = \act'(\lf_j)\Er$, the terms present inside the expectation are exactly those in the learning term of Eq.~\eqref{eq:qm_up}.

%% file: appendix/init_cond.tex
\section{Initial conditions and symmetric teacher}\label{app:init_cond}

In this work we have constructed teacher matrices $ \W^* \in \R^{\hidt\times\inp} $ in order to have 
\begin{equation}
\label{eq:sym_teacher}
    \prs =  \frac{ \w^{*\top}_r \w^{*}_{s} }{d} = \delta_{rs} \;,
\end{equation}
where $ \w^{*}_r \equiv [\W^*]_r \in \R^\inp $ is the $r$-th row of the matrix $\W^*$. We have started by sampling $\hidt$ vectors of dimension $\inp$ uniformly on a ball of radius $\sqrt{\inp}$. Then we constructed an orthonormal basis using singular value decomposition. 

The initial student weights $\W^0 \in \R^{\hids\times\inp}$ were taken as 
\begin{equation}
\W^0 = \bm{A} \W^*     \;,
\end{equation}
with each row of $\bm{A} \in \R^{\hids \times \hidt}$ sampled uniformly on a ball of radius one. We acknowledge choosing initial student weights as linear combinations of the teacher can be artificial and shrinks the first plateau, but our focus on this work was the specialization phase. Nevertheless, this choice and Eq.~(\ref{eq:sym_teacher}) are particularly suitable to theoretical analysis. Once $\hidt$ and $\hids$ are fixed, the dimension $\inp$ can be varied without changing $ \Q^0$, $\M^0$ and $\P$, thereby removing any influence of different initial conditions for different $\inp$ and providing the reader better visualization on the learning curves. To clarify this point, consider the $j$-th row $ \w^{0}_j \equiv [\W^0]_j \in \R^\inp $ of $\W^0 $:
\begin{equation}
    \w^{0}_j   =  \sum_{r=1}^{\hidt} a_{jr} \w_{r}^*  \;,
\end{equation}
with $ a_{jr} \equiv [ \bm{A}  ]_{jr} $. Using Eq.~\eqref{eq:sym_teacher} one can write
\begin{equation}
\qjl^0 = \frac{\w^{0\top}_j \w^{0}_l }{d} = \sum_{r,r' =1}^{\hidt} a_{jr} a_{j r'}   \underbrace{\frac{\w^{*\top}_r \w^{*}_{r'} }{d}}_{=\delta_{ r r'} } = \sum_{r = 1}^{\hidt} a_{jr} a_{l r} \;.
\end{equation}
Similarly,
\begin{equation}
\mjr^0 =\frac{\w^{0\top}_j \w^{*}_r }{d}  =  a_{jr} \;.
\end{equation}
Thus once $\bm{A}$ is fixed, the input dimension $\inp$ can be varied without affecting the initial conditions. We chose to sample $ \bm{a}_j \equiv [ \bm{A}  ]_j \in \R^\hidt $ on a ball of radius one both to introduce some randomness on the initialization and to keep the initial parameters bounded by one.

We stress that we use these initial conditions to make the data comparable for varying dimension $d$ in the numerical illustrations. Our conclusions do not depend on this particular choice of initial conditions. If one simply takes random initialization $ \w_j \sim \gauss  (\w_j |\bm{0},\Id ) $ for each $j$, the full picture we have presented in this manuscript remains unchanged. In Figure \ref{fig:plot_d_init_not_inf} we present an example of curves within the blue region (see Section \ref{sec:main} for the characterization of this regime) with unconstrained Gaussian initialization. Dots represent simulations, while solid lines are obtained by integration of the ODEs given by Eqs.~(\ref{eq:qmode0}), with initial conditions adjusted to match simulations.

Although varying the initial population risk with $\inp$ slightly changes the exact position where the specialization transition starts, the particular initial conditions adopted in this work do not affect whether the specialization transition takes place or not, comparing to unconstrained Gaussian initialization.

\begin{figure}[H]
\begin{center}
\centerline{\includegraphics[width=0.45\textwidth]{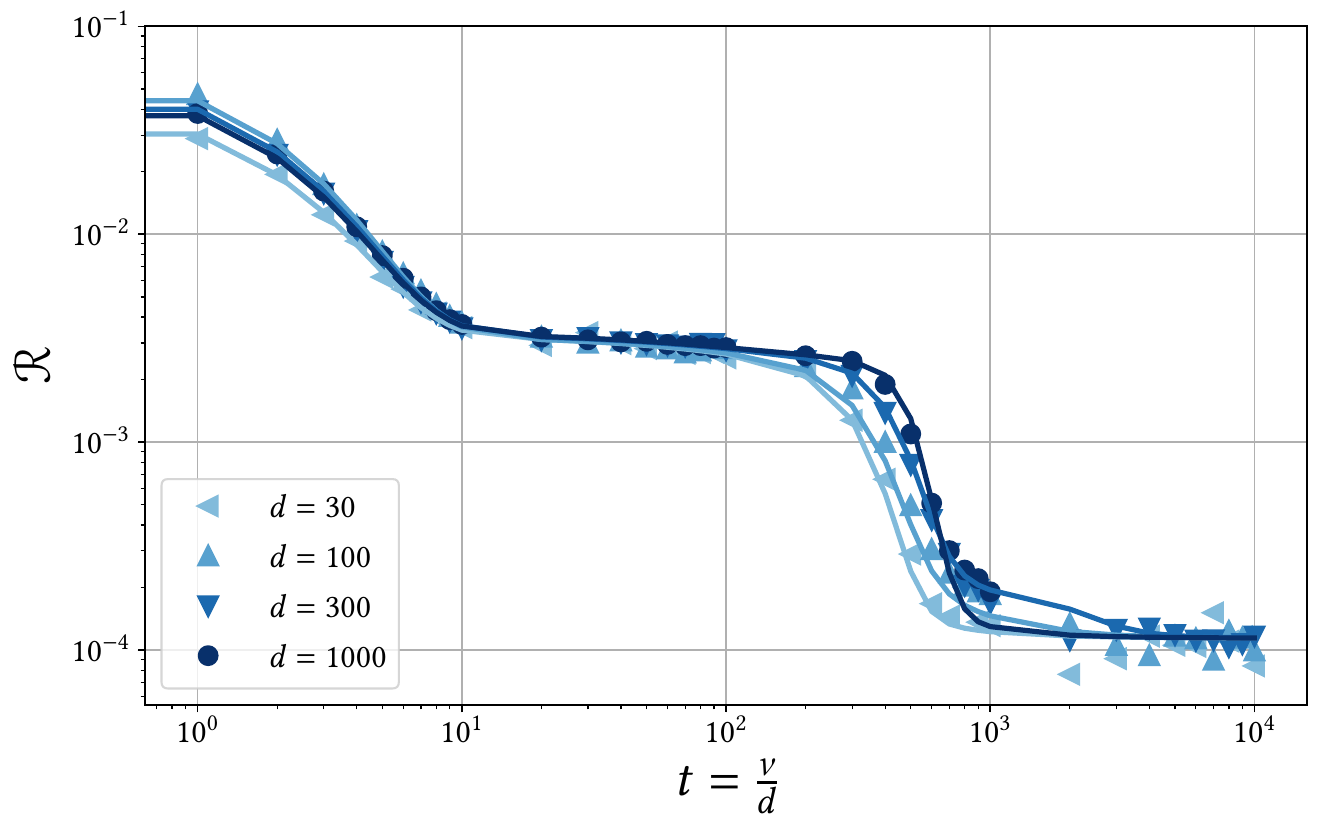}}
\caption{Population risk dynamics for $\kap=\del=0$ (Saad \& Solla scaling) : $\hids_0 = 8$, $\hidt=4$, $\prs = \delta_{rs}$. Initialization: $ \w_j \sim \gauss  (\w_j |\bm{0},\Id ) $ for $j = 1, ..., \hids_0$. Activation function: $\act(x) = \erf(x/\sqrt{2})$. Data distribution: $\Prob (\x) =  \gauss(\x | \bm{0}, \Id )$. Dots represent simulations, while solid lines are obtained by integration of the ODEs given by Eqs.~(\ref{eq:qmode0}), with initial conditions adjusted to match simulations. Observe the difference on the initialization for different $\inp$.}
\label{fig:plot_d_init_not_inf}
\end{center}
\vskip -0.2in
\end{figure}